\documentclass[10pt,journal,compsoc]{IEEEtran}

\usepackage{amsmath,amssymb,amsfonts}
\usepackage{algorithmic}
\usepackage{algorithm}
\usepackage{graphicx}
\usepackage{textcomp}
\usepackage{xcolor}
\usepackage{booktabs}
\usepackage{multirow}
\usepackage{hyperref}
\usepackage{cleveref}
\usepackage{bm}
\usepackage{subcaption}

\newcommand{\vect}[1]{\mathbf{#1}}
\newcommand{\mat}[1]{\mathbf{#1}}
\newcommand{\R}{\mathbb{R}}

\newcommand{\nobsp}{NObSP}

\usepackage{booktabs}

\usepackage{amsthm}
\usepackage{tikz}
\usetikzlibrary{shapes.geometric, arrows.meta, positioning, calc, fit, backgrounds}

\newtheorem{definition}{Definition}[section]
\newtheorem{proposition}{Proposition}[section]
\newtheorem{theorem}{Theorem}[section]

\usepackage{xcolor}

\begin{document}

\title{NObSP: Functional Decomposition of Neural Networks via Oblique Subspace Projections}

\author{Alexander~Caicedo, V\'ictor~De~La~Hoz, and~Santiago~Alf\'erez
\IEEEcompsocitemizethanks{\IEEEcompsocthanksitem A. Caicedo is with the Department of Electronic Engineering, Pontificia Universidad Javeriana, Bogot\'a, Colombia.
\protect\\
E-mail: acaicedod@javeriana.edu.co
\IEEEcompsocthanksitem V. De La Hoz is with Pontificia Universidad Javeriana, Bogot\'a, Colombia.
\IEEEcompsocthanksitem S. Alf\'erez is with the Department of Mathematics, Universitat Polit\`ecnica de Catalunya, Barcelona, Spain.
\protect\\
E-mail: santiago.alferez@upc.edu
\IEEEcompsocthanksitem This work is part of a research project funded by the Ministry of Science and Innovation of Spain, with reference PID2023-146261OB-I00.
}}

\IEEEtitleabstractindextext{

\begin{abstract}
Understanding how deep neural networks make decisions remains a fundamental challenge. We present \nobsp{} (Nonlinear Oblique Subspace Projections), a framework that decomposes predictions into explicit per feature contribution functions and an interaction residual. \nobsp{} exploits the linear final layer of a trained network and uses oblique projections in sample space to reduce double counting when learned feature subspaces overlap, thereby supporting both local explanations and global functional analysis. We establish connections to functional ANOVA and the Kolmogorov-Arnold representation theorem and derive an efficient partial regression algorithm for out of sample evaluation. For convolutional networks, \nobsp{}-CAM produces class activation maps without backward passes after a one time calibration. Experiments on tabular and vision benchmarks show faithfulness comparable to established attribution methods. On a synthetic benchmark with known component functions, \nobsp{} obtains a Function Reproduction Score of 0.989, compared with 0.966 for KernelSHAP and 0.922 for Integrated Gradients. On TinyImageNet, contribution vector embeddings improve mean nearest neighbor class purity from 0.654 for raw activations to 0.713 and reduce mean neighbor distance by more than half. These results indicate that \nobsp{} complements scalar attribution methods by recovering functional contribution profiles with separable positive and negative evidence.
\end{abstract}


\vspace{0.5em}
\begin{center}
\footnotesize\itshape
This work has been submitted to the IEEE for possible publication. Copyright may be transferred without notice, after which this version may no longer be accessible.
\end{center}}

\maketitle

\IEEEdisplaynontitleabstractindextext

\IEEEpeerreviewmaketitle

\section{Introduction}
\label{sec:introduction}

\IEEEPARstart{T}{he} deployment of deep neural networks in high risk domains including medical diagnosis, autonomous systems, financial decisions, and legal reasoning, has created an urgent need for interpretability methods that explain how these models arrive at their predictions~\cite{doshi2017towards,lipton2018mythos}. Regulatory frameworks such as the European Union's General Data Protection Regulation increasingly require explanations for automated decisions~\cite{goodman2017european}. Yet neural networks remain largely opaque; while we can observe their inputs and outputs, the internal computations that transform one to the other are difficult to interpret.
Existing interpretability methods face fundamental limitations. Attribution methods such as SHAP~\cite{lundberg2017unified} and LIME~\cite{ribeiro2016why} provide scalar importance scores, but knowing that feature $x_k$ contributes a specific amount to the prediction reveals far less than knowing its functional relationship $g_k(x_k)$. Functional decomposition methods such as Partial Dependence Plots (PDP)~\cite{friedman2001greedy} and Accumulated Local Effects (ALE) ~\cite{apley2020visualizing} recover functional forms but struggle with correlated features, accumulate numerical errors, and lack efficient out-of-sample extension, PDP requires re-averaging over training samples for each new query while ALE produces binned approximations that do not naturally extrapolate beyond the observed feature range. For convolutional networks, Class Activation Mapping methods~\cite{selvaraju2017grad,muhammad2020eigen} highlight spatial regions but provide no insight into the functional relationship between pixel values and predictions. Most critically, existing methods are designed for either local or global interpretation, requiring practitioners to employ multiple techniques for comprehensive understanding.

This paper presents \nobsp{} (Nonlinear Oblique Subspace Projections), a mathematically principled framework that addresses these limitations by decomposing neural network predictions into interpretable functional components. The main idea behind this decomposition is architectural; in a trained neural network, all nonlinearity resides in how inputs transform to the penultimate layer representation $\mathbf{z} = f_{\text{NN}\to L-1}(\mathbf{x})$, while the final layer acting on this latent variable, $\hat{y} = \mathbf{w}^\top \mathbf{z} + b$, is linear. This means the prediction vector $\hat{\vect{y}} = \mat{Z}\vect{w} + b\vect{1}$ is linear in the representations, so oblique projection decomposition in sample space $\R^N$ is tractable. The subspaces $\mathcal{V}_k$ traced out by each isolated input feature encode the learned nonlinear transformations, passing an input where only feature $k$ is active through the network reveals exactly how the network processes that feature.

\nobsp{} decomposes predictions into additive functional components:
\begin{equation}
\hat{y}(\mathbf{x}) = \sum_{k=1}^{d} g_k(x_k) + G_{\text{int}}(\mathbf{x}) + b,
\label{eq:nobsp_decomp}
\end{equation}
where $g_k(x_k)$ represents the nonlinear contribution of feature $k$, and $G_{\text{int}}$ captures interaction effects. The critical challenge is that learned subspaces are generally not orthogonal, the network may use overlapping representations for different features, a phenomenon known as superposition~\cite{elhage2022superposition}. Classical ANOVA decomposition requires orthogonality, but \nobsp{} resolves this through oblique subspace projections~\cite{behrens1994signal}, which project onto a target subspace along, not perpendicular to, a reference subspace. This isolates the unique contribution of each feature while properly accounting for representational overlap.

The relationship between \nobsp{} and existing methods indicates its position in the interpretability landscape. Both \nobsp{} and SHAP connect to functional ANOVA decomposition. SHAP distributes interaction effects among features via Shapley values, yielding scalars that answer \emph{how much does feature $k$ contribute?} \nobsp{} instead keeps functional terms separate, yielding complete functions that answer \emph{how does feature $k$ contribute across its range?} These perspectives are complementary. Compared to Partial Dependence Plots, \nobsp{} avoids extrapolation into unlikely feature combinations; compared to Accumulated Local Effects, it does not rely on derivative estimation and accumulation. Unlike gradient based CAM methods that operate per sample, \nobsp{} provides a unified framework for both local explanations, individual predictions, and global analysis, feature effects across the dataset.

Beyond attribution, \nobsp{} contribution vectors $\mathbf{g}(\mathbf{x}) = [g_1(x_1), \ldots, g_d(x_d)]^\top$ induce a geometrically meaningful embedding space. Preliminary experiments show that this embedding improves mean class purity in nearest neighbor retrieval (0.71 vs 0.65 for raw activations on TinyImageNet, with median purity 1.0), suggesting that \nobsp{} captures the classifier's decision geometry rather than merely the feature space. This property enables applications in uncertainty quantification and similarity based explanation retrieval.

This paper makes the following contributions:

\begin{enumerate}
\item \textbf{Theoretical Framework}: We establish formal connections between \nobsp{}, the Kolmogorov-Arnold representation theorem, and functional ANOVA decomposition, stated under explicit assumptions.

\item \textbf{Efficient Algorithm}: We introduce a partial regression formulation that reduces computational complexity from $O(N^2)$ to $O(d_z)$ for out-of-sample evaluation, where $d_z$ is the penultimate layer dimension, enabling practical application to large-scale models.

\item \textbf{\nobsp{}-CAM}: We extend \nobsp{} to convolutional neural networks, producing class activation maps grounded in the oblique projection framework.

\item \textbf{Embedding Geometry Analysis}: We provide preliminary evidence that \nobsp{} contribution vectors form a useful retrieval embedding, improving neighborhood class purity by 6 percentage points over raw activations and reducing mean neighbor distance by more than half on TinyImageNet, without requiring ground truth labels at inference.

\item \textbf{Comprehensive Evaluation}: We benchmark \nobsp{} on both tabular and vision tasks, showing faithfulness comparable to established attribution methods while adding a complementary form of interpretability unavailable from scalar attributions, explicit per feature contribution functions.
\end{enumerate}

The remainder of this paper is organized as follows. Section~\ref{sec:related} reviews related work. Section~\ref{sec:theory} presents the theoretical foundations. Section~\ref{sec:method} describes the \nobsp{} framework and its extension to convolutional neural networks, including the embedding geometry analysis in Subsection~\ref{sec:embedding}. Section~\ref{sec:experiments} details experimental setup. Section~\ref{sec:results} presents results. Section~\ref{sec:discussion} discusses findings and limitations. Section~\ref{sec:conclusion} concludes.

\section{Related Work}
\label{sec:related}

This section reviews the interpretability landscape, positioning \nobsp{} relative to existing approaches. Table~\ref{tab:method_comparison} summarizes the key differences across methods.

\begin{table}[t]
\centering
\caption{Comparison of interpretability methods. \nobsp{} provides functional forms with unified local and global interpretation while accounting for overlap between learned feature subspaces.}
\label{tab:method_comparison}
\footnotesize
\setlength{\tabcolsep}{4pt}
\begin{tabular}{lccccc}
\toprule
Method & Output & Local & Global & Functional & Correlation \\
\midrule
LIME & Scalar & \checkmark & -- & -- & -- \\
SHAP & Scalar & \checkmark & Agg. & -- & \checkmark \\
IG & Scalar & \checkmark & -- & -- & -- \\
PDP & Curve & -- & \checkmark & \checkmark & -- \\
ALE & Curve & -- & \checkmark & \checkmark & \checkmark \\
GradCAM & Heatmap & \checkmark & -- & -- & N/A \\
\textbf{\nobsp{}} & \textbf{Function} & \checkmark & \checkmark & \checkmark & \checkmark \\
\bottomrule
\end{tabular}
\end{table}

\subsection{Gradient Based Attribution Methods}

Gradient based methods leverage the differentiability of neural networks to assign importance scores to input features. Saliency Maps~\cite{simonyan2014deep} compute the gradient of the output with respect to the input, measuring local sensitivity but suffering from gradient saturation and providing only first-order information. Integrated Gradients~\cite{sundararajan2017axiomatic} addresses some limitations by accumulating gradients along a path from a baseline to the input, satisfying desirable axioms including completeness and sensitivity. DeepLIFT~\cite{shrikumar2017learning} propagates activation differences through the network with improved computational efficiency. However, all these methods share baseline dependency and provide scalar attributions rather than functional forms.

For convolutional networks, Class Activation Mapping (CAM)~\cite{zhou2016learning} exploits global average pooling to produce spatial heatmaps. GradCAM~\cite{selvaraju2017grad} generalizes this to arbitrary architectures by using gradients to compute channel importance weights. Subsequent variants include GradCAM++~\cite{chattopadhay2018grad}, which uses weighted combinations of positive gradients, Score-CAM~\cite{wang2020score}, which perturbs inputs with upsampled activation maps, and Eigen-CAM~\cite{muhammad2020eigen}, which uses principal component analysis on feature maps. All gradient based methods share a fundamental limitation, they provide point attributions or spatial heatmaps but not the functional relationship $g_k(x_k)$ between input values and predictions.

\subsection{Perturbation Based Attribution Methods}

Perturbation based methods measure feature importance by observing prediction changes when inputs are modified. LIME~\cite{ribeiro2016why} fits a sparse linear model around a specific input, weighting samples by proximity and encouraging sparsity. While model agnostic, LIME provides only local linear approximations that may poorly represent nonlinear behaviour.

SHAP~\cite{lundberg2017unified} computes Shapley values from cooperative game theory, satisfying desirable properties including local accuracy, missingness, and consistency. The relationship between SHAP and \nobsp{} merits careful analysis. Both connect to functional ANOVA decomposition but in different ways. SHAP distributes interaction effects among participating features according to Shapley values, producing a single scalar per feature that aggregates main effects and allocated interactions. \nobsp{} instead keeps these terms separate, $g_k(x_k)$ captures the main effect as a function, while interaction terms remain distinct. Thus, SHAP answers ``how much does feature $k$ contribute?'' while \nobsp{} answers ``how does feature $k$ contribute as a function of its value?''

\subsection{Functional Decomposition Methods}

Functional decomposition methods aim to recover the relationship between individual features and predictions, moving beyond scalar attributions.

Partial Dependence Plots (PDP)~\cite{friedman2001greedy} estimate the marginal effect of feature $k$ by averaging predictions over the empirical distribution of other features. While intuitive, PDP has critical limitations. When features are correlated, averaging over the marginal distribution includes combinations that may never occur in the data, potentially producing misleading extrapolations. Additionally, PDP lacks efficient out-of-sample extension, evaluating a new point requires re-averaging over all training samples, demanding $O(N)$ model evaluations per query.

Individual Conditional Expectation (ICE) plots~\cite{goldstein2015peeking} address heterogeneity by displaying curves for individual observations rather than averaging. However, ICE inherits the extrapolation problem of PDP and can produce cluttered visualizations.

Accumulated Local Effects (ALE)~\cite{apley2020visualizing} addresses the correlation problem by accumulating local derivatives using conditional expectations. By avoiding extrapolation into unlikely feature combinations, ALE handles correlated features more appropriately than PDP. However, the accumulation of finite differences introduces numerical errors, the method exhibits sensitivity near points with discontinuous derivatives, and it produces binned approximations that do not naturally extrapolate beyond the observed feature range.

Functional ANOVA decomposition~\cite{hooker2007generalized} provides the theoretical foundation for these approaches. When features are correlated, the standard definition becomes problematic: conditional expectations conflate the effect of one feature with effects of correlated features. \nobsp{} offers advantages over these functional methods: unlike PDP, it does not average over marginal distributions and thus avoids extrapolation; unlike ALE, it does not rely on derivative estimation and accumulation; and unlike standard ANOVA implementations, it uses oblique rather than orthogonal projections to reduce double counting when learned feature subspaces overlap. Furthermore, \nobsp{} precomputes coefficient vectors during calibration, enabling efficient $O(d_z)$ out-of-sample evaluation with a single forward pass.

\subsection{Interpretable by Design Models}

An alternative to post-hoc interpretation is to constrain model architecture for inherent interpretability. Generalized Additive Models (GAMs)~\cite{hastie1990generalized} restrict the learned function to additive form, guaranteeing interpretability but sacrificing flexibility to model interactions. Neural Additive Models (NAMs)~\cite{agarwal2021neural} implement GAMs with neural networks, learning each univariate function as a separate subnetwork, and GA$^2$M models~\cite{lou2012intelligible} extend additive models with selected pairwise interactions.

The fundamental trade-off is between interpretability and flexibility. Interpretable by design models constrain what can be learned, potentially sacrificing predictive performance. \nobsp{} avoids this trade-off by extracting interpretable decompositions from unconstrained networks that have already learned arbitrary functions.

\subsection{Subspace and Projection Methods}

The mathematical foundation of \nobsp{} lies in subspace projection theory from signal processing. Oblique projections were introduced by Behrens and Scharf~\cite{behrens1994signal} for signal separation when signal and interference subspaces are non-orthogonal. Given a target subspace $\mathcal{V}_k$ and reference subspace $\mathcal{V}_{(k)}$ with $\mathcal{V}_k \cap \mathcal{V}_{(k)} = \{\vect{0}\}$, the oblique projector uniquely decomposes any vector into components lying in each subspace. Unlike orthogonal projection, which minimizes distance to the subspace, oblique projection moves parallel to the reference subspace until reaching the target.

Caicedo et al.~\cite{caicedo2019nonlinear} extended oblique projections to nonlinear function decomposition in Least Squares Support Vector Machines. The present work extends this framework to deep neural networks, where the penultimate layer representation plays the role of the feature map, and introduces computational improvements that reduce complexity from $O(N^2)$ to $O(d_z)$.

\subsection{Evaluation Methods}

Standardized evaluation of attribution methods remains challenging. The Quantus framework~\cite{hedstrom2023quantus} provides a comprehensive toolkit implementing multiple metrics across categories including faithfulness, robustness, and complexity.

Faithfulness Correlation measures the correlation between feature attribution importance and the actual change in model prediction when features are progressively removed. Higher values indicate that the method correctly identifies which features matter most for the prediction. Sparseness measures how concentrated the attributions are, with higher values indicating more focused explanations that highlight specific regions rather than diffusing attention across the input. We employ Faithfulness Correlation for tabular attributions and Sparseness in both domains; for class activation maps we adopt a region level deletion protocol (Section~\ref{sec:exp_vision}), because pixel level perturbation is mismatched to the spatial granularity of CAMs. No single evaluation metric is neutral across attribution families, so we report several and state each metric's bias where relevant.

\section{Theoretical Foundations}
\label{sec:theory}

This section establishes the mathematical foundations of \nobsp{}, connecting oblique subspace projections to functional ANOVA decomposition and neural network interpretation.

\subsection{Oblique Subspace Projections}

Oblique projections generalize orthogonal projections to handle non-orthogonal subspaces. Rather than projecting perpendicular to the target subspace, oblique projection moves parallel to a reference subspace until reaching the target, isolating the component that lies in the target and \textit{cannot be explained} by the reference.

\begin{definition}[Oblique Projector]
\label{def:oblique}
Let $\mathcal{V} = \mathcal{V}_k \oplus \mathcal{V}_{(k)}$ be a direct sum decomposition of vector space $\mathcal{V}$, where $\mathcal{V}_k$ is the target subspace and $\mathcal{V}_{(k)}$ is the reference subspace. The oblique projector onto $\mathcal{V}_k$ along $\mathcal{V}_{(k)}$ is the linear operator $\mat{P}_{k/(k)}$ satisfying:
\begin{enumerate}
    \item $\mat{P}_{k/(k)} \vect{v} \in \mathcal{V}_k$ for all $\vect{v} \in \mathcal{V}$ (range condition),
    \item $\mat{P}_{k/(k)} \vect{v} = \vect{v}$ for all $\vect{v} \in \mathcal{V}_k$ (identity on target),
    \item $\mat{P}_{k/(k)} \vect{v} = \vect{0}$ for all $\vect{v} \in \mathcal{V}_{(k)}$ (null on reference).
\end{enumerate}
\end{definition}

The third condition distinguishes oblique from orthogonal projection: vectors in the reference subspace project to zero, even if they have nonzero orthogonal projection onto the target. This ensures that shared components between subspaces are not double counted.

Given matrices $\mat{A}_k$ and $\mat{A}_{(k)}$ whose columns span $\mathcal{V}_k$ and $\mathcal{V}_{(k)}$ respectively, the oblique projector admits the closed form expression
\begin{equation}
\mat{P}_{k/(k)} = \mat{A}_k \left(\mat{A}_k^T \mat{Q}_{(k)} \mat{A}_k\right)^{\dagger} \mat{A}_k^T \mat{Q}_{(k)},
\label{eq:oblique_proj}
\end{equation}
where $\dagger$ denotes the Moore-Penrose pseudoinverse and
\begin{equation}
\mat{Q}_{(k)} = \mat{I} - \mat{A}_{(k)}\left(\mat{A}_{(k)}^T\mat{A}_{(k)}\right)^{\dagger}\mat{A}_{(k)}^T
\label{eq:q_matrix}
\end{equation}
is the orthogonal projector onto the null space of $\mathcal{V}_{(k)}$.

\subsection{Connection to Functional ANOVA}

The Hoeffding or Functional ANOVA decomposition~\cite{hoeffding1948class,hooker2007generalized} provides that any square-integrable function $f: \R^d \to \R$ of independent random variables admits a unique decomposition:
\begin{equation}
f(\vect{x}) = f_0 + \sum_{k=1}^{d} f_k(x_k) + \sum_{k<m} f_{km}(x_k, x_m) + \ldots + f_{1,\ldots,d}(\vect{x}),
\label{eq:anova}
\end{equation}
where components satisfy orthogonality constraints ensuring that main effects capture only variance due to individual variables.

The connection to oblique projections emerges when we represent functions by their evaluations over a sample. Given $N$ observations, the vector $\vect{y} = [f(\vect{x}^{(1)}), \ldots, f(\vect{x}^{(N)})]^T$ represents function evaluations in $\R^N$. The subspace $\mathcal{V}_k$ represents functions depending only on $x_k$, while $\mathcal{V}_{(k)}$ represents functions of all variables except $x_k$. Classical ANOVA uses orthogonal projections, requiring orthogonality among functional components. When representation subspaces are non-orthogonal, as commonly occurs in neural networks where learned representations exhibit overlap, oblique projections provide correct attribution.

\begin{theorem}[\nobsp{} as Generalized ANOVA]
\label{thm:anova}
Let $\vect{y} \in \R^N$ be function evaluations and $\{\mathcal{V}_k\}_{k=1}^d$ be subspaces representing single-variable functions. The oblique projection decomposition
\begin{equation}
\vect{y} = \sum_{k=1}^{d} \mat{P}_{k/(k)}\vect{y} + \vect{r}
\end{equation}
provides an ANOVA-like decomposition that (i) reduces to classical ANOVA when subspaces are orthogonal, (ii) correctly handles non-orthogonal subspaces via oblique projection, and (iii) provides unique attribution of main effects.
\end{theorem}

\begin{proof}[Proof sketch]
When subspaces are orthogonal, $\mat{Q}_{(k)}$ acts as identity on $\mathcal{V}_k$, and \eqref{eq:oblique_proj} reduces to the orthogonal projector. For non-orthogonal subspaces, Definition~\ref{def:oblique} ensures $\mat{P}_{k/(k)}$ extracts only the component uniquely attributable to $\mathcal{V}_k$. The full proof appears in Supplementary Material A.
\end{proof}

\subsection{Decomposition in Neural Networks}

Consider a neural network $f: \R^d \to \R^c$ with $L$ layers. The main idea is that all nonlinearity resides in the mapping to the penultimate layer $\vect{z} = f_{NN \to L-1}(\vect{x}) \in \R^{d_z}$, while the final layer is linear:
\begin{equation}
\hat{y}_j = \vect{w}_j^T \vect{z} + b_j,
\label{eq:linear_final}
\end{equation}
where $\vect{w}_j$ are the weights for output $j$ and $b_j$ is the bias term. This linearity enables a projection based decomposition of the prediction vector in sample space $\R^N$, with per sample evaluation through the penultimate representation.

For each input feature $k$, we construct subspaces by passing isolated inputs through the network. Define $\mat{X}_k$ as the input matrix with only feature $k$ active, others set to zero or a baseline value, and $\mat{X}_{(k)}$ with all features except $k$ active. The corresponding penultimate representations are:
\begin{align}
\mat{Z}_k &= f_{NN \to L-1}(\mat{X}_k), \\
\mat{Z}_{(k)} &= f_{NN \to L-1}(\mat{X}_{(k)}).
\end{align}
Each matrix $\mat{Z}_k \in \R^{N \times d_z}$ embeds the feature $k$ representations into sample space. Its column space $\mathcal{V}_k = \text{col}(\tilde{\mat{Z}}_k) \subseteq \R^N$ is the subspace of $N$-dimensional vectors representable as linear combinations of the penultimate layer activations from isolated feature $k$ inputs. Analogously, $\mathcal{V}_{(k)} = \text{col}(\tilde{\mat{Z}}_{(k)}) \subseteq \R^N$.

\begin{theorem}[Prediction Decomposition]
\label{thm:decomp}
For a neural network with linear final layer $\hat{y}_j = \vect{w}_j^T \vect{z} + b_j$, the centered prediction vector decomposes as
\begin{equation}
  \tilde{\vect{y}}_j
  = \hat{\vect{y}}_j-b_j = \sum_{k=1}^{d} \mat{P}_{k/(k)}\,\tilde{\vect{y}}_j + \vect{r}_j
  = \sum_{k=1}^{d} \vect{g}_{k,j} + \vect{r}_j,
\label{eq:decomp}
\end{equation}
where each $\vect{g}_{k,j} = \mat{P}_{k/(k)}\,\tilde{\mat{Z}}\,\vect{w}_j \in \R^N$ is the contribution vector of feature~$k$ to output~$j$. Through the partial regression equivalence (Theorem~\ref{thm:beta}),
\begin{align}
  \vect{g}_{k,j} &= \tilde{\mat{Z}}_k\,\bm{\Omega}_k\,\vect{w}_j,
  \notag \\
  \bm{\Omega}_k
  &= \bigl(\tilde{\mat{Z}}_k^T \mat{Q}_{(k)}\,\tilde{\mat{Z}}_k\bigr)^{\dagger}
    \tilde{\mat{Z}}_k^T \mat{Q}_{(k)}\,\tilde{\mat{Z}}
  \;\in\; \R^{d_z\times d_z},
\label{eq:omega}
\end{align}
yielding per sample functional contributions that depend only on~$x_k$:
\begin{equation}
  g_{k,j}(x_k^{(i)})
  = \tilde{\vect{z}}_k^{(i)\top}\,\bm{\Omega}_k\,\vect{w}_j.
\label{eq:per_sample}
\end{equation}
\end{theorem}

\begin{proof}[Proof sketch]
By linearity of the final layer, $\tilde{\vect{y}}_j = \tilde{\mat{Z}}\vect{w}_j$. Applying the oblique projection decomposition from Theorem~\ref{thm:anova} to the centered predictions yields~\eqref{eq:decomp}. By the partial regression equivalence, $\mat{P}_{k/(k)}\,\tilde{\vect{y}}_j = \tilde{\mat{Z}}_k\,\bm{\beta}_{k,j}$ with $\bm{\beta}_{k,j} = \bm{\Omega}_k\,\vect{w}_j$. The $i$-th entry $g_{k,j}(x_k^{(i)}) = \tilde{\vect{z}}_k^{(i)\top}\bm{\Omega}_k\,\vect{w}_j$ depends only on $x_k^{(i)}$ because $\tilde{\vect{z}}_k^{(i)} = f_{NN\to L-1}(\vect{x}_k^{(i)}) - \bar{\vect{z}}_k$ is computed from the isolated input. See Supplementary Material A for details.
\end{proof}

\subsection{Uniqueness}

\begin{theorem}[Uniqueness of Decomposition]
\label{thm:unique}
Given a trained network $f_{NN}$, calibration data $\mat{X}$, and baseline choice, assume $\mathcal{V}_k \cap \mathcal{V}_{(k)} = \{\vect{0}\}$ for all $k$. Then the \nobsp{} decomposition is unique. When the intersection is nontrivial, the regularized estimator of Algorithm~\ref{alg:beta} defines a unique solution of minimum norm, which we use in practice.
\end{theorem}

\begin{proof}[Proof sketch]
Uniqueness follows from three observations: (1) subspaces $\mathcal{V}_k \subseteq \R^N$ are uniquely determined by forward passes of isolated inputs through the fixed network; (2) oblique projectors are unique under the stated intersection assumption; (3) contributions $g_{k,j}(x_k) = \tilde{\vect{z}}_k^T \bm{\beta}_{k,j}$ follow uniquely from the unique projectors and fixed final layer weights, since $\bm{\beta}_{k,j} = (\tilde{\mat{Z}}_k^T \mat{Q}_{(k)} \tilde{\mat{Z}}_k)^{\dagger} \tilde{\mat{Z}}_k^T \mat{Q}_{(k)} \tilde{\vect{y}}_j$ is uniquely determined. Different interpretability methods yield different attributions because they define ``contribution'' differently, this reflects methodological choice, not non-uniqueness within \nobsp{}.
\end{proof}

\subsection{Connection to Kolmogorov-Arnold Representation}

The Kolmogorov-Arnold theorem~\cite{kolmogorov1957representation,arnold1957functions} establishes that any continuous multivariate function can be expressed as sums of univariate functions composed with addition. Neural networks implement an analogous structure, each layer applies univariate activations to linear combinations. The penultimate representation $\vect{z}$ encapsulates these learned compositions. \nobsp{} extracts this learned representation by tracing computational pathways activated by each feature, with subspace $\mat{Z}_k$ encoding the learned transformation of $x_k$ analogously to the inner functions in the Kolmogorov-Arnold representation.

\subsection{Baseline Choice}

Setting inactive features to zero is well justified when inputs are normalized to zero mean, zero represents the population average, analogous to baselines in SHAP and Integrated Gradients. The interpretation becomes ``contribution of feature $k$ relative to its average value.''

\section{\nobsp{} Framework}
\label{sec:method}

This section presents the \nobsp{} algorithm for neural networks and its extension to convolutional architectures. Figure~\ref{fig:overview} provides a visual overview of the complete pipeline.

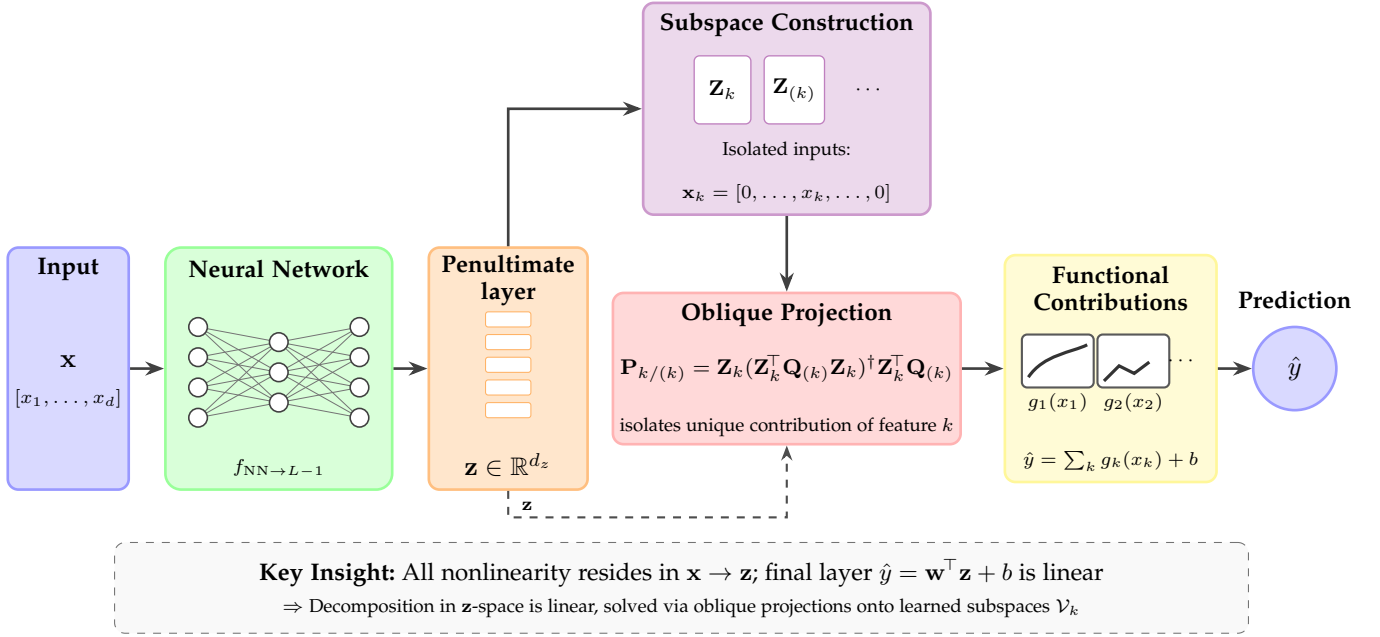
\begin{figure*}[t]
\centering
\begin{tikzpicture}[
    every node/.append style={text=black},
    inputnode/.style={rectangle, rounded corners=4pt, minimum width=1.6cm, minimum height=3.2cm,
                      fill=blue!15, draw=blue!40, line width=1pt},
    networknode/.style={rectangle, rounded corners=6pt, minimum width=3.0cm, minimum height=3.2cm,
                        fill=green!15, draw=green!40, line width=1pt},
    penultnode/.style={rectangle, rounded corners=4pt, minimum width=2.1cm, minimum height=3.2cm,
                       fill=orange!20, draw=orange!50, line width=1pt},
    subspacenode/.style={rectangle, rounded corners=4pt, minimum width=3.8cm, minimum height=2.8cm,
                         fill=violet!15, draw=violet!40, line width=1pt},
    projnode/.style={rectangle, rounded corners=4pt, minimum width=4.6cm, minimum height=2.0cm,
                     fill=red!15, draw=red!30, line width=1pt},
    funcnode/.style={rectangle, rounded corners=4pt, minimum width=2.8cm, minimum height=3.0cm,
                     fill=yellow!20, draw=yellow!50, line width=1pt},
    outputnode/.style={circle, minimum size=1.1cm, fill=blue!15, draw=blue!40, line width=1pt},
    matrixbox/.style={rectangle, rounded corners=2pt, minimum width=0.75cm, minimum height=0.9cm,
                      fill=white, draw=violet!50, line width=0.5pt, font=\footnotesize},
    funcbox/.style={rectangle, rounded corners=2pt, minimum width=0.95cm, minimum height=0.7cm,
                    fill=white, draw=black!65, line width=0.8pt},
    arrow/.style={-{Stealth[length=2.5mm]}, line width=1pt, color=black!75},
    dasharrow/.style={-{Stealth[length=2mm]}, line width=0.8pt, color=black!70, dashed},
    nodelabel/.style={font=\small, color=black},
    blocktitle/.style={font=\small\bfseries, color=black},
    sublabel/.style={font=\scriptsize, color=black},
]

\node[inputnode] (input) at (0, 0) {};
\node[blocktitle] at ($(input.north) + (0, -0.3)$) {Input};
\node[nodelabel, font=\bfseries] at ($(input.center) + (0, 0.1)$) {$\mathbf{x}$};
\node[sublabel] at ($(input.center) + (0, -0.4)$) {$[x_1, \ldots, x_d]$};

\node[networknode, right=0.45cm of input] (network) {};
\node[blocktitle] at ($(network.north) + (0, -0.3)$) {Neural Network};

\foreach \yl in {0.55, 0.15, -0.25, -0.65}
    \foreach \ym in {0.35, -0.05, -0.45}
        \draw[black!55, line width=0.35pt]
            ($(network.west) + (0.45, \yl)$) -- ($(network.center) + (0, \ym)$);
\foreach \ym in {0.35, -0.05, -0.45}
    \foreach \yr in {0.55, 0.15, -0.25, -0.65}
        \draw[black!55, line width=0.35pt]
            ($(network.center) + (0, \ym)$) -- ($(network.east) + (-0.45, \yr)$);

\foreach \y in {0.55, 0.15, -0.25, -0.65} {
    \filldraw[fill=white, draw=black!75, line width=0.6pt] ($(network.west) + (0.45, \y)$) circle (3.5pt);
}
\foreach \y in {0.35, -0.05, -0.45} {
    \filldraw[fill=white, draw=black!75, line width=0.6pt] ($(network.center) + (0, \y)$) circle (3.5pt);
}
\foreach \y in {0.55, 0.15, -0.25, -0.65} {
    \filldraw[fill=white, draw=black!75, line width=0.6pt] ($(network.east) + (-0.45, \y)$) circle (3.5pt);
}

\node[sublabel] at ($(network.south) + (0, 0.3)$) {$f_{\mathrm{NN}\to L-1}$};

\node[penultnode, right=0.45cm of network] (penult) {};
\node[blocktitle, align=center] at ($(penult.north) + (0, -0.45)$) {Penultimate\\layer};

\foreach \y in {0.55, 0.25, -0.05, -0.35, -0.65} {
    \fill[white, draw=orange!60, rounded corners=1pt] ($(penult.center) + (-0.3, \y)$) rectangle ++(0.6, 0.2);
}

\node[nodelabel, font=\bfseries] at ($(penult.south) + (0, 0.35)$) {$\mathbf{z} \in \mathbb{R}^{d_z}$};

\node[subspacenode, above right=0.4cm and 0.7cm of penult] (subspace) {};
\node[blocktitle] at ($(subspace.north) + (0, -0.3)$) {Subspace Construction};

\node[matrixbox] (Zk) at ($(subspace.center) + (-0.85, 0.25)$) {$\mathbf{Z}_k$};
\node[matrixbox] (Znk) at ($(subspace.center) + (0.1, 0.25)$) {$\mathbf{Z}_{(k)}$};
\node[sublabel] at ($(subspace.center) + (1.1, 0.25)$) {$\cdots$};

\node[sublabel] at ($(subspace.center) + (0, -0.55)$) {Isolated inputs:};

\node[sublabel, align=center] at ($(subspace.south) + (0, 0.3)$)
    {$\mathbf{x}_k = [0,\ldots,x_k,\ldots,0]$};

\node[projnode, below=1.0cm of subspace] (proj) {};
\node[blocktitle] at ($(proj.north) + (0, -0.3)$) {Oblique Projection};
\node[nodelabel, font=\footnotesize] at ($(proj.center) + (0, 0.0)$)
    {$\mathbf{P}_{k/(k)} = \mathbf{Z}_k(\mathbf{Z}_k^{\!\top} \mathbf{Q}_{(k)} \mathbf{Z}_k)^{\dagger} \mathbf{Z}_k^{\!\top} \mathbf{Q}_{(k)}$};
\node[sublabel, align=center] at ($(proj.south) + (0, 0.25)$)
    {isolates unique contribution of feature~$k$};

\node[funcnode, right=0.55cm of proj] (func) {};
\node[blocktitle, align=center] at ($(func.north) + (0, -0.45)$)
    {Functional\\Contributions};

\node[funcbox] (g1) at ($(func.center) + (-0.70, 0.1)$) {};
\draw[black!80, line width=1.2pt] ($(g1.south west) + (0.1, 0.15)$)
    .. controls ++(0.2, 0.3) and ++(-0.2, -0.1) .. ($(g1.north east) + (-0.1, -0.15)$);
\node[sublabel] at ($(g1.south) + (0, -0.18)$) {$g_1(x_1)$};

\node[funcbox] (g2) at ($(func.center) + (0.30, 0.1)$) {};
\draw[black!80, line width=1.2pt] ($(g2.south west) + (0.1, 0.1)$)
    -- ++(0.2, 0.2) -- ++(0.2, -0.1) -- ++(0.2, 0.15);
\node[sublabel] at ($(g2.south) + (0, -0.18)$) {$g_2(x_2)$};

\node[sublabel] at ($(func.center) + (0.95, 0.1)$) {$\cdots$};

\node[sublabel] at ($(func.south) + (0, 0.30)$) {$\hat{y} = \sum_k g_k(x_k) + b$};

\node[outputnode, right=0.45cm of func] (output) {};
\node[nodelabel, font=\bfseries] at (output.center) {$\hat{y}$};
\node[blocktitle] at ($(output.north) + (0, 0.35)$) {Prediction};

\draw[arrow] (input.east) -- (network.west);
\draw[arrow] (network.east) -- (penult.west);
\draw[arrow] (penult.north) |- (subspace.west);
\draw[arrow] (subspace.south) -- (proj.north);
\draw[arrow] (proj.east) -- (func.west);
\draw[arrow] (func.east) -- (output.west);

\draw[dasharrow] (penult.south)
    -- ++(0, -0.35)
    -| (proj.south);
\node[sublabel, font=\scriptsize\bfseries] at ($(penult.south) + (0.25, -0.2)$) {$\mathbf{z}$};

\node[rectangle, rounded corners=4pt, draw=black!55, dashed, fill=gray!5,
      minimum width=15cm, minimum height=1.2cm,
      below=1.2cm of $(input.south)!0.5!(output.south)$, anchor=north] (insight) {};

\node[align=center, font=\small, color=black] at (insight.center) {
    \textbf{Key Insight:} All nonlinearity resides in $\mathbf{x} \to \mathbf{z}$;
    final layer $\hat{y} = \mathbf{w}^{\!\top} \mathbf{z} + b$ is linear\\[2pt]
    \scriptsize $\Rightarrow$ Decomposition in $\mathbf{z}$-space is linear,
    solved via oblique projections onto learned subspaces $\mathcal{V}_k$
};

\end{tikzpicture}
\caption{Overview of the \nobsp{} framework. An input $\mathbf{x}$ passes through the neural network to produce penultimate layer activations $\mathbf{z}$. Subspaces $\mathbf{Z}_k$ are constructed by passing isolated inputs (where only feature $k$ is active) through the network. Oblique projection onto $\mathcal{V}_k$ along $\mathcal{V}_{(k)}$ isolates the unique contribution of each feature, yielding functional forms $g_k(x_k)$ that sum to the prediction. The linear final layer enables this decomposition to be solved exactly.}
\label{fig:overview}
\end{figure*}

\subsection{Problem Formulation}

Consider a trained neural network $f_{NN}: \R^d \to \R^c$ mapping $d$-dimensional inputs to $c$ outputs. For a dataset $\mat{X} \in \R^{N \times d}$ with predictions $\hat{\mat{Y}} \in \R^{N \times c}$, the objective is to decompose predictions into interpretable functional components:
\begin{equation}
\hat{y}_j(\vect{x}) = b_j + \sum_{k=1}^{d} g_{k,j}(x_k) + \sum_{k<m} g_{km,j}(x_k, x_m) + \ldots
\end{equation}
where $g_{k,j}(x_k)$ represents the functional contribution of feature $k$ to output $j$.

\subsection{Subspace Construction}

The construction of feature subspaces requires a baseline representing the absence of a feature. We assume inputs are normalized to zero mean and unit variance, a standard preprocessing step for neural networks. Under this normalization, setting inactive features to zero positions them at their population mean, providing a baseline with clear statistical interpretation.

For each feature $k$, we construct the isolated input matrix $\mat{X}_k \in \R^{N \times d}$ where only the $k$-th column contains observed values and all other columns are zero. Similarly, the reference input matrix $\mat{X}_{(k)}$ contains all features except $k$. Passing these through the network to the penultimate layer yields representation matrices:
\begin{align}
\mat{Z}_k &= f_{NN \to L-1}(\mat{X}_k) \in \R^{N \times d_z}, \\
\mat{Z}_{(k)} &= f_{NN \to L-1}(\mat{X}_{(k)}) \in \R^{N \times d_z},
\end{align}
where $d_z$ denotes the penultimate layer dimension. After mean centering via $\tilde{\mat{Z}}_k = \mat{M}\mat{Z}_k$ with centering matrix $\mat{M} = \mat{I} - \frac{1}{N}\vect{1}\vect{1}^T$, the column spaces define the feature subspace $\mathcal{V}_k = \text{col}(\tilde{\mat{Z}}_k) \subseteq \R^N$ and reference subspace $\mathcal{V}_{(k)} = \text{col}(\tilde{\mat{Z}}_{(k)}) \subseteq \R^N$, respectively.

\subsection{Efficient Computation via Partial Regression}

Direct computation of the oblique projector $\mat{P}_{k/(k)}$ requires $O(N^2)$ storage and $O(N^3)$ computation per feature, rendering it impractical for large datasets. We introduce an equivalent formulation based on partial regression that reduces storage complexity to $O(d_z)$.

The main idea behind is that oblique projection can be reformulated as sequential partial regressions. Rather than explicitly constructing projection matrices, we compute coefficient vectors $\bm{\beta}_{k,j} \in \R^{d_z}$ such that contributions are obtained as $\hat{\vect{y}}_{k,j} = \tilde{\mat{Z}}_k \bm{\beta}_{k,j}$.

\begin{algorithm}[t]
\caption{Efficient Partial Regression Formulation}
\label{alg:beta}
\begin{algorithmic}[1]
\REQUIRE $\tilde{\mat{Z}}_k$ (target basis), $\tilde{\mat{Z}}_{(k)}$ (reference basis), $\hat{\vect{y}}_j$ (predictions), $\lambda$ (regularization)
\ENSURE $\bm{\beta}_{k,j}$ (coefficient vector)
\STATE $\mat{A} \gets \left(\tilde{\mat{Z}}_{(k)}^T\tilde{\mat{Z}}_{(k)} + \lambda\mat{I}\right)^{-1} \tilde{\mat{Z}}_{(k)}^T$
\STATE $\hat{\vect{y}}_{res} \gets \hat{\vect{y}}_j - \tilde{\mat{Z}}_{(k)} \mat{A} \hat{\vect{y}}_j$ \COMMENT{Remove reference subspace from predictions}
\STATE $\tilde{\mat{Z}}_{res} \gets \tilde{\mat{Z}}_k - \tilde{\mat{Z}}_{(k)} \mat{A} \tilde{\mat{Z}}_k$ \COMMENT{Remove reference subspace from target}
\STATE $\bm{\beta}_{k,j} \gets \left(\tilde{\mat{Z}}_{res}^T\tilde{\mat{Z}}_{res} + \lambda\mat{I}\right)^{-1} \tilde{\mat{Z}}_{res}^T \hat{\vect{y}}_{res}$
\RETURN $\bm{\beta}_{k,j}$
\end{algorithmic}
\end{algorithm}

\begin{theorem}[Equivalence of Partial Regression Formulation]
\label{thm:beta}
The contribution $\tilde{\mat{Z}}_k \bm{\beta}_{k,j}$ computed via Algorithm~\ref{alg:beta} coincides with the oblique projection of $\hat{\vect{y}}_j$ onto $\mathcal{V}_k$ along $\mathcal{V}_{(k)}$:
\begin{equation}
\tilde{\mat{Z}}_k \bm{\beta}_{k,j} = \mat{P}_{k/(k)}\hat{\vect{y}}_j.
\end{equation}
\end{theorem}

\begin{proof}[Proof sketch]
The residuals from Steps 2--3 satisfy $\hat{\vect{y}}_{res} = \mat{Q}_{(k)} \hat{\vect{y}}_j$ and $\tilde{\mat{Z}}_{res} = \mat{Q}_{(k)} \tilde{\mat{Z}}_k$, where $\mat{Q}_{(k)}$ is the orthogonal complement projector. Using the idempotence $\mat{Q}_{(k)}^2 = \mat{Q}_{(k)}$ and symmetry $\mat{Q}_{(k)}^T = \mat{Q}_{(k)}$, we obtain $\bm{\beta}_{k,j} = \left(\tilde{\mat{Z}}_k^T \mat{Q}_{(k)} \tilde{\mat{Z}}_k\right)^{\dagger} \tilde{\mat{Z}}_k^T \mat{Q}_{(k)} \hat{\vect{y}}_j$. Substituting into the contribution expression yields the oblique projector form. Full details appear in the supplementary material.
\end{proof}

The regularization parameter $\lambda > 0$ provides numerical stability when subspace bases are ill-conditioned, with equivalence holding exactly as $\lambda \to 0$.

This formulation offers substantial computational advantages. Calibration requires solving least squares problems of dimension $d_z \times d_z$, yielding complexity $O(Nd_z^2 + d_z^3)$ per feature versus $O(N^2 d_z + N^3)$ for direct projection. Since $d_z \ll N$ in typical applications, this represents a substantial reduction.

More importantly, the coefficient vectors enable efficient out-of-sample evaluation. For a new input $\vect{x}^*$, the contribution of feature $k$ to output $j$ is:
\begin{equation}
g_{k,j}(x_k^*) = \vect{z}_k^{*T} \bm{\beta}_{k,j},
\end{equation}
where $\vect{z}_k^* = f_{NN \to L-1}(\vect{x}_k^*)$ is the penultimate representation of the isolated input. This requires only a single forward pass plus an inner product, yielding $O(d_z)$ complexity per feature, no projection matrices need be stored, and no reference to the calibration data is required at inference time.

\subsection{Extension to CNNs: \nobsp{}-CAM}

For convolutional neural networks, \nobsp{} is adapted to generate class activation maps by treating feature map channels as variables to be decomposed.

Consider a CNN with convolutional backbone producing spatial feature maps $\mat{F} \in \R^{C \times H \times W}$, followed by global average pooling $\vect{z} = \text{GAP}(\mat{F}) \in \R^C$ and linear classifier $\hat{\vect{y}} = \mat{W}\vect{z} + \vect{b}$. Treating each channel $c$ as a feature yields:
\begin{equation}
\hat{y}_q = b_q + \sum_{c=1}^{C} g_{c,q}(z_c),
\end{equation}
where $q$ indexes the output class.

Unlike gradient based methods that operate per sample, \nobsp{}-CAM requires calibration across multiple samples to estimate oblique projections. The procedure extracts pooled features and predictions from a calibration set, computes $\bm{\beta}$ coefficients via Algorithm~\ref{alg:beta}, and caches them for inference. Calibration requirements depend on penultimate layer dimension; in our experiments, 1,000--5,000 samples provided stable estimates across architectures with 440--1,280 dimensional representations.

For a test image with spatial feature maps $\mat{F}$ and target class $q$, the class activation map is:
\begin{equation}
\text{CAM}_q(h, w) = \sum_{c=1}^{C} F_{c,h,w} \cdot \max(g_{c,q}(z_c), 0),
\label{eq:nobsp_cam}
\end{equation}
using positive coefficients only. This follows the convention established by GradCAM~\cite{selvaraju2017grad}, which applies ReLU to retain only features with positive influence on the class score.

The \nobsp{} framework naturally produces both positive and negative coefficients, where negative values indicate channels providing evidence against the target class. While this decomposition offers potential interpretive value, the present work focuses exclusively on positive contributions for direct comparability with existing methods. Analysis of negative evidence maps remains an avenue for future investigation.

Once calibrated, \nobsp{}-CAM inference requires only a single forward pass to obtain spatial features, followed by a weighted sum with cached coefficients, comparable to GradCAM's cost without requiring a backward pass.

The framework generalizes beyond the standard GAP architecture. For CNNs employing more complex classifiers, such as one or more fully connected layers after the convolutional backbone, \nobsp{} applies identically by extracting $\vect{z}$ from the penultimate layer of the classifier head rather than directly from pooled features. The only requirement is that the final layer mapping $\vect{z}$ to outputs remains linear, which holds for virtually all classification architectures.

Regarding calibration sample requirements, while the $d_z \times d_z$ matrix inversions in Algorithm~\ref{alg:beta} nominally suggest $N \geq d_z$ samples, the actual requirement is determined by the effective rank of the feature subspaces $\mathcal{V}_k$, not the ambient dimension $d_z$. Caicedo et al.~\cite{caicedo2019nonlinear} observed that decomposition error converges once the number of samples exceeds the maximum rank of the subspace matrices; in their experiments this was 45 samples despite much higher ambient dimensionality.  This phenomenon reflects the low intrinsic dimensionality of learned representations. Neural networks typically encode information on low-dimensional manifolds within the activation space. Consequently, regularization ($\lambda > 0$) combined with the inherently low-rank structure of learned subspaces enables stable coefficient estimation with far fewer samples than $d_z$. We employed calibration sets of 1,000--5,000 samples to ensure robustness across architectures, though substantially smaller sets may suffice when the effective subspace rank is low.

\subsection{Embedding Space and Similarity Retrieval}
\label{sec:embedding}

Beyond functional decomposition, the \nobsp{} framework induces a geometrically meaningful embedding space with applications for similarity based retrieval and uncertainty quantification. This subsection develops these connections.

\subsubsection{Contribution Vector Embedding}

For a given input $\vect{x}$, \nobsp{} decomposition produces a contribution vector:
\begin{equation}
\vect{g}(\vect{x}) = [g_1(x_1), g_2(x_2), \ldots, g_d(x_d)]^T \in \R^d,
\end{equation}
encoding how each input feature contributes to the prediction. This vector defines an embedding where geometry reflects the functional structure of predictions rather than raw feature similarity.

A key property derives from the continuity of contribution functions. For networks with continuous activations (sigmoid, tanh, GELU), the functions $g_k(x_k)$ are smooth; for piecewise linear activations (ReLU), they are piecewise linear and continuous. In either case, the mapping $\vect{x} \mapsto \vect{g}(\vect{x})$ preserves neighborhood structure.

\begin{proposition}[Continuity of \nobsp{} Embedding]
\label{prop:continuity}
Let $f_{NN}: \R^d \to \R$ be a neural network with continuous activation functions and linear output layer. Then \nobsp{} contribution functions $g_k: \R \to \R$ are continuous for all $k$, and consequently the embedding $\vect{g}: \R^d \to \R^d$ is continuous.
\end{proposition}

The proof follows from observing that $g_k(x_k) = \tilde{\vect{z}}_k^{\top} \bm{\Omega}_k \vect{w}$ is a composition of continuous mappings, the network forward pass on the isolated input, centering, and multiplication by the fixed matrices $\bm{\Omega}_k$ and $\vect{w}$. This continuity implies that \nobsp{} embedding provides a representation space where distance reflects functional similarity in how predictions are formed.

\subsubsection{Similarity Based Explanation Retrieval}

\nobsp{} embedding facilitates explanation through similarity. Given a test prediction, retrieve training samples whose contribution vectors are most similar.

\begin{definition}[\nobsp{} Similarity]
For inputs $\vect{x}$ and $\vect{x}'$ with contribution vectors $\vect{g}(\vect{x})$ and $\vect{g}(\vect{x}')$, \nobsp{} similarity is:
\begin{equation}
s(\vect{x}, \vect{x}') = \frac{\vect{g}(\vect{x})^T \vect{g}(\vect{x}')}{\|\vect{g}(\vect{x})\| \cdot \|\vect{g}(\vect{x}')\|}.
\end{equation}
\end{definition}

In a deployment setting, the target class $c$ used to compute contribution vectors is the model's predicted class $\hat{y}$, requiring no ground-truth labels. This makes the embedding computable at inference time for any new input.

For a test input $\vect{x}^*$, retrieving the $K$ nearest neighbors according to \nobsp{} similarity provides interpretive value. The prediction for $\vect{x}^*$ is formed similarly to predictions for the retrieved neighbors, offering exemplar based explanations aligned with human reasoning.

\subsubsection{Neighborhood Class Purity}

Inspecting neighbor class labels reveals whether a prediction is well supported by consistent training evidence or arises from an ambiguous region near decision boundaries.

\begin{definition}[Neighborhood Class Purity]
For test input $\vect{x}^*$ with predicted class $\hat{y}^*$ and $K$ nearest neighbors $\mathcal{N}_K(\vect{x}^*)$:
\begin{equation}
\rho(\vect{x}^*) = \frac{1}{K} \sum_{\vect{x}^{(i)} \in \mathcal{N}_K(\vect{x}^*)} \mathrm{I}[y^{(i)} = \hat{y}^*],
\end{equation}
where $y^{(i)}$ is the true label of training sample $\vect{x}^{(i)}$.
\end{definition}

A purity value $\rho(\vect{x}^*) \approx 1$ indicates the prediction is supported by consistent training evidence, while $\rho(\vect{x}^*) \approx 1/C$ (where $C$ is the number of classes) suggests the sample lies in a highly ambiguous region. This provides uncertainty quantification complementing softmax confidence, as predictions with high softmax confidence but low neighborhood purity signal potential overconfidence.

\subsubsection{Computational Considerations}

Similarity based retrieval requires computing distances between the test contribution vector and all training embeddings. For large training sets, approximate nearest neighbor methods (locality-sensitive hashing, hierarchical navigable small world graphs) provide efficient retrieval with sublinear complexity in $N$. Storage requirement is $O(Nd)$ for training embeddings, typically modest compared to storing training data. Embeddings are computed once after model training and cached for subsequent retrieval. For \nobsp{}-CAM, contribution vectors have dimension $C$ (number of channels), enabling efficient storage even for large-scale datasets.

\section{Experimental Setup}
\label{sec:experiments}

We evaluate \nobsp{} on both tabular and vision benchmarks to demonstrate its versatility across data modalities. For tabular data, we compare against established attribution methods on classification and regression tasks. For computer vision, we extend \nobsp{} to \nobsp{}-CAM and compare against gradient based class activation mapping methods.

\subsection{Tabular Experiments}
\label{sec:exp_tabular}

\subsubsection{Datasets}

We evaluate on seven classification and two regression datasets spanning diverse domains. Table~\ref{tab:tabular_datasets} summarizes the datasets. Feature counts reflect encoded features after preprocessing, numeric columns are standardized to zero mean and unit variance, categorical variables are one-hot encoded, and missing values are imputed with the median (numeric) or mapped to an explicit category (categorical).

\begin{table}[t]
\centering
\caption{Tabular benchmark datasets. All datasets use 80/20 train/test splits with three random seeds for statistical reliability.}
\label{tab:tabular_datasets}
\footnotesize
\setlength{\tabcolsep}{3pt}
\begin{tabular}{@{}lcccc@{}}
\toprule
Dataset & Samples & Features & Task & Domain \\
\midrule
\multicolumn{5}{@{}l}{\textit{Classification}} \\
Iris & 150 & 4 & 3-class & Botany \\
Breast Cancer & 569 & 30 & Binary & Medical \\
Wine & 178 & 13 & 3-class & Chemistry \\
Adult & 5,000 & 124 & Binary & Census \\
German Credit & 1,000 & 61 & Binary & Finance \\
Bank Marketing & 5,000 & 51 & Binary & Finance \\
Covertype & 10,000 & 54 & 7-class & Ecology \\
\midrule
\multicolumn{5}{@{}l}{\textit{Regression}} \\
Diabetes & 442 & 10 & Continuous & Medical \\
California Housing & 20,640 & 8 & Continuous & Economics \\
\bottomrule
\end{tabular}
\end{table}

\subsubsection{Model Architecture}

For classification, we use a feedforward network with two hidden layers (20 and 200 neurons), ReLU activations, and softmax output. For regression, hidden layers of 128 and 64 neurons with linear output. All models are trained with Adam (learning rate $3 \times 10^{-3}$, weight decay $10^{-4}$) for 35--40 epochs with early stopping.

\subsubsection{Comparison Methods}

We compare \nobsp{} against two established attribution methods:
\begin{itemize}
\item \textbf{Integrated Gradients (IG)}~\cite{sundararajan2017axiomatic}: Accumulates gradients along a path from baseline (training mean) to input using 25 integration steps.
\item \textbf{KernelSHAP}~\cite{lundberg2017unified}: Approximates Shapley values via weighted linear regression on feature coalitions, using 50 background samples.
\end{itemize}

For \nobsp{}, we calibrate on the training split and evaluate on the test split. To ensure comparability with baseline-perturbation metrics, we report $\Delta$-baseline attributions: $\text{attr}_k(\vect{x}) = g_k(x_k) - g_k(x_k^{\text{baseline}})$.

\subsubsection{Evaluation Metrics}

Standard perturbation faithfulness evaluates scalar attribution rankings, whereas \nobsp{} produces feature wise functional contributions. We therefore complement a conventional attribution metric with a function level metric for the setting where ground truth component functions are known.

\textbf{Attribution Faithfulness Correlation} measures the Spearman correlation between mean absolute attributions per feature and prediction change when each feature is set to baseline. Higher values indicate better alignment between attributed importance and actual feature influence. We report it for comparability with prior XAI literature; it evaluates scalar rankings, not functional recovery.

\textbf{Function Reproduction Score (FRS)} directly evaluates functional recovery on a synthetic additive benchmark with known components, $y = |x_1| + x_2^3 + e^{x_3} + \sin(2 x_4) + 0 \cdot x_5 + \eta$, where $x_5$ is a null feature (1{,}000 samples, noise level 0.01, 80/20 split, seeds 0--2, scores computed on 200 held out test samples per seed). Because additive decompositions are identifiable only up to per feature offsets, both the true component $f_k$ and the estimated contribution are centered feature wise before scoring, and the direct FRS for feature $k$ is the coefficient of determination of the estimate against the centered ground truth. We also report the normalized RMSE, the mean absolute attribution assigned to the null feature (false positive magnitude), and the reconstruction R$^2$ of summed contributions against the model output. Reconstruction measures additive output accounting, not recovery of the individual component functions.

\subsection{Vision Experiments}
\label{sec:exp_vision}

\subsubsection{Model Architectures}

We evaluate \nobsp{}-CAM across four CNN architectures representing diverse design philosophies. Table~\ref{tab:models} summarizes the architectures. The penultimate layer dimension directly affects calibration time for \nobsp{}-CAM.

\begin{table}[t]
\centering
\caption{CNN architectures evaluated for \nobsp{}-CAM.}
\label{tab:models}
\footnotesize
\setlength{\tabcolsep}{3pt}
\begin{tabular}{@{}lccc@{}}
\toprule
Model & $d_z$ & Parameters & Type \\
\midrule
ResNet-18~\cite{he2016deep} & 512 & 11.7M & Residual \\
MobileNetV2~\cite{sandler2018mobilenetv2} & 1280 & 3.4M & Inverted Residual \\
EfficientNet-B0~\cite{tan2019efficientnet} & 1280 & 4.7M & Compound Scaling \\
RegNetY-400MF~\cite{radosavovic2020designing} & 440 & 4.3M & Regularized Design \\
\bottomrule
\end{tabular}
\end{table}

\subsubsection{Datasets}

We evaluate on six vision datasets spanning general object recognition, fine-grained classification, and medical imaging. Table~\ref{tab:vision_datasets} summarizes the datasets.

\begin{table}[t]
\centering
\caption{Vision benchmark datasets. Training regime indicates whether models were trained from scratch or fine-tuned from ImageNet pretrained weights.}
\label{tab:vision_datasets}
\footnotesize
\setlength{\tabcolsep}{2pt}
\begin{tabular}{@{}lccccc@{}}
\toprule
Dataset & Classes & Train/Val & Resolution & Regime \\
\midrule
CIFAR-10 & 10 & 50K/10K & $32^2$ or $32^2 \to 128^2$* & Scratch \\
CIFAR-100 & 100 & 50K/10K & $32^2$ or $32^2 \to 128^2$* & Scratch \\
TinyImageNet & 200 & 100K/10K & $64^2 \to 224^2$ & Fine-tune \\
Imagenette & 10 & 9K/4K & $224^2$ & Fine-tune \\
STL-10 & 10 & 5K/8K & $96^2 \to 224^2$ & Fine-tune \\
DermaMNIST & 7 & 7K/1K & $28^2 \to 224^2$ & Fine-tune \\
\bottomrule
\end{tabular}
\end{table}

DermaMNIST~\cite{yang2023medmnist}, derived from the HAM10000 dataset~\cite{tschandl2018ham10000}, contains dermatoscopic images of pigmented skin lesions classified into seven diagnostic categories. This dataset provides a clinically relevant benchmark where interpretability is crucial for deployment.

For CIFAR-10 and CIFAR-100, models are trained from scratch. For larger resolution datasets and DermaMNIST, we fine-tune from ImageNet pretrained weights.

*The CIFAR resolution depends on the checkpoint architecture. Small-input ResNet-18 checkpoints are evaluated at native $32 \times 32$ resolution with a modified downsampling stem, whereas ImageNet-style compact backbones (MobileNetV2, EfficientNet-B0, RegNetY-400MF) use $128 \times 128$ resized inputs to avoid overly coarse final feature maps for CAM generation. All methods compared within a dataset and model row use the same checkpoint, preprocessing, target layer, and image indices, so method comparisons are unaffected by this choice.

Table~\ref{tab:model_accuracy} reports the accuracy of the trained checkpoints analyzed in the CAM benchmark. Accuracies are adequate for studying explanations of competent models; we make no competitive accuracy claims.

\begin{table}[t]
\centering
\caption{Top-1 validation accuracy of the checkpoints used in the CAM benchmark, as evaluated by the benchmark harness on held out validation subsets of 320--1{,}000 images per configuration.}
\label{tab:model_accuracy}
\footnotesize
\begin{tabular}{@{}lccc@{}}
\toprule
Model & CIFAR-10 & TinyImageNet & DermaMNIST \\
\midrule
ResNet-18 & 0.948 & 0.759 & 0.856 \\
MobileNetV2 & 0.852 & 0.763 & 0.834 \\
EfficientNet-B0 & 0.910 & 0.800 & 0.838 \\
RegNetY-400MF & 0.914 & 0.803 & 0.837 \\
\bottomrule
\end{tabular}
\end{table}

\subsubsection{Comparison Methods}

We compare \nobsp{}-CAM against three class activation mapping methods:
\begin{itemize}
\item \textbf{GradCAM}~\cite{selvaraju2017grad}: Computes channel weights using gradients of class score with respect to feature maps, applying ReLU to retain positive contributions.
\item \textbf{GradCAM++}~\cite{chattopadhay2018grad}: Uses weighted combinations of positive partial derivatives for improved multi-instance localization.
\item \textbf{Eigen-CAM}~\cite{muhammad2020eigen}: Applies PCA to feature maps, using the first principal component as importance weights.
\end{itemize}

Following GradCAM convention, we focus on positive contributions (\nobsp{}-CAM$^+$) for primary comparisons.

\subsubsection{Evaluation Metrics}

We evaluate on synchronized sample subsets of 200 images per dataset and model configuration, so that all CAM methods within a configuration are applied to identical image indices and method comparisons are paired.

\textbf{Region Deletion AUC} ranks heatmap patches by mean attribution, progressively masks the highest ranked patches with a zero baseline, and integrates the normalized drop in target class probability over 10 perturbation steps (patch size 32, 16, or 8 pixels for inputs of at least 224, 128, or 64 pixels, respectively). Higher values indicate that the highlighted regions carry more evidence for the target class. We adopt this region level metric because pixel level faithfulness correlation is mismatched to the spatial granularity of CAMs, which are coarse by construction; under that pixel level metric all methods score near zero. The faithfulness correlation results are reported in the supplementary material for completeness.

\textbf{Sparseness}, computed with the Quantus framework~\cite{hedstrom2023quantus} on synchronized subsets of 20--50 images per configuration, measures attribution concentration: $\text{S} = 1 - \|\vect{a}\|_1 / (\sqrt{n}\|\vect{a}\|_2)$. Higher values indicate more focused explanations. Sparseness describes the shape of an explanation, not its correctness.

\subsection{Implementation Details}
\label{sec:implementation}

All experiments use Python 3.10 with PyTorch 2.5.1 (CUDA 12.4), Captum 0.8.0, and Quantus 0.6.0 on a single NVIDIA RTX 3090 (24 GB). Sparseness and the faithfulness correlation are computed with Quantus; Region Deletion AUC and the tabular metrics are implemented in our benchmark harness. Vision calibration uses mixed precision (bf16) with a batched Cholesky solver. All experiments are repeated across three random seeds (0, 1, 2), and all regularized solves use $\lambda = 10^{-4}$. A reference implementation of NObSP is publicly available at \url{https://github.com/santialferez/nobsp_lib}.

For tabular experiments, \nobsp{} is calibrated on the training split and evaluated on the held out test split.

For vision experiments, \nobsp{}-CAM calibration employs the partial regression formulation (Algorithm~\ref{alg:beta}) with Cholesky solver and $\lambda = 10^{-4}$. Calibration is performed on validation sets (1{,}000--5{,}000 images per configuration) with coefficients cached for inference. We disclose that the images used for metric evaluation are drawn from the same validation split used for \nobsp{}-CAM calibration (the training split is consumed by model training); evaluation images are a synchronized subset, identical for all CAM methods, and the gradient based baselines require no calibration. This shared split could in principle favor \nobsp{}-CAM, so we note that its Region Deletion AUC results remain slightly below GradCAM, which bounds any such advantage. Table~\ref{tab:calibration_time} presents calibration times by architecture.

\begin{table}[t]
\centering
\caption{Calibration time by penultimate layer dimension. Once calibrated, inference requires only a forward pass plus matrix multiplication.}
\label{tab:calibration_time}
\footnotesize
\begin{tabular}{@{}lccc@{}}
\toprule
Model & Feature Dim & Median Time & Max Time \\
\midrule
RegNetY-400MF & 440 & $<$1s & $<$1s \\
ResNet-18 & 512 & 6.7s & $\sim$30s \\
EfficientNet-B0 & 1280 & 14min & 45min \\
MobileNetV2 & 1280 & 25min & 2.7h \\
\bottomrule
\end{tabular}
\end{table}

Calibration time scales with penultimate dimension: 440--512 dimensional models complete in seconds, while 1280-dimensional models require 14--25 minutes. This one-time cost amortizes over inference, where \nobsp{}-CAM requires only a forward pass followed by matrix multiplication, comparable to GradCAM without a backward pass.

\section{Results}
\label{sec:results}

This section presents quantitative and qualitative evaluation of \nobsp{} on tabular and vision benchmarks, along with embedding analysis demonstrating the geometric properties of contribution vectors.

\subsection{Tabular Data Results}
\label{sec:results_tabular}

\subsubsection{Attribution-Faithfulness Correlation}

Table~\ref{tab:tabular_faithfulness} presents attribution faithfulness correlation results. This metric measures alignment between attributed feature importances and actual prediction changes when features are masked.

\begin{table}[t]
\centering
\caption{Attribution Faithfulness Correlation on tabular datasets (higher is better). Values averaged across three seeds. Best per dataset in \textbf{bold}.}
\label{tab:tabular_faithfulness}
\footnotesize
\setlength{\tabcolsep}{4pt}
\begin{tabular}{@{}lccc@{}}
\toprule
Dataset & \nobsp{} & IG & KernelSHAP \\
\midrule
\multicolumn{4}{@{}l}{\textit{Classification}} \\
Iris & \textbf{1.000} & \textbf{1.000} & \textbf{1.000} \\
Breast Cancer & 0.946 & \textbf{0.981} & 0.922 \\
Wine & 0.916 & \textbf{0.991} & 0.980 \\
Adult & 0.861 & \textbf{0.986} & 0.819 \\
German Credit & 0.764 & \textbf{0.962} & 0.756 \\
Bank Marketing & 0.847 & \textbf{0.978} & 0.836 \\
Covertype & 0.853 & \textbf{0.954} & 0.871 \\
\midrule
\multicolumn{4}{@{}l}{\textit{Regression}} \\
Diabetes & 0.912 & \textbf{0.987} & 0.956 \\
California Housing & 0.878 & \textbf{0.991} & 0.934 \\
\bottomrule
\end{tabular}
\end{table}

Integrated Gradients achieves the highest faithfulness across datasets, expected given IG directly optimizes for baseline-to-input attribution. \nobsp{} achieves competitive performance, particularly on simpler datasets (Iris, Breast Cancer), and outperforms KernelSHAP on several classification datasets.

\subsubsection{Functional Decomposition Quality}

Beyond numerical metrics, \nobsp{} provides interpretable functional forms $g_k(x_k)$ revealing how each feature affects predictions across its range. Figure~\ref{fig:synthetic_decomposition} illustrates this on the synthetic additive regression with known ground truth $y = |x_1| + x_2^3 + e^{x_3} + \sin(2 x_4) + 0 \cdot x_5 + \eta$. \nobsp{} accurately recovers the absolute value, cubic, exponential, and sinusoidal components, while correctly assigning the null feature $x_5$ a negligible contribution.

Table~\ref{tab:frs} quantifies this recovery with the Function Reproduction Score. \nobsp{} attains the highest direct FRS (0.989, against 0.966 for KernelSHAP and 0.922 for Integrated Gradients), the lowest normalized error, and the smallest false positive attribution on the null feature. KernelSHAP and Integrated Gradients reconstruct the model output exactly (reconstruction R$^2$ of 1.000) because both are constructed around additive output accounting, yet their pointwise recovery of the true component functions is weaker. \nobsp{}'s reconstruction R$^2$ of 0.975 reflects the interaction residual that its main effect decomposition excludes by construction. Per feature FRS scores appear in the supplementary material. This function level evaluation is the primary quantitative evidence for \nobsp{}'s central claim, when the evaluation target is recovery of the true component functions, \nobsp{} recovers them more accurately than scalar attribution methods.

\begin{table}[t]
\centering
\caption{Function Reproduction Score on the synthetic additive benchmark (mean $\pm$ SD over seeds 0--2). Direct FRS R$^2$ measures pointwise recovery of the centered ground truth components; Null abs.\ is the mean absolute attribution on the null feature $x_5$; Recon.\ R$^2$ measures reconstruction of the model output by summed contributions. Bold marks the numerical best per column.}
\label{tab:frs}
\footnotesize
\setlength{\tabcolsep}{2pt}
\begin{tabular}{@{}lcccc@{}}
\toprule
Method & FRS R$^2$ $\uparrow$ & NRMSE $\downarrow$ & Null abs. $\downarrow$ & Recon. R$^2$ $\uparrow$ \\
\midrule
\nobsp{} & \textbf{0.989 $\pm$ 0.001} & \textbf{0.098 $\pm$ 0.002} & \textbf{0.009 $\pm$ 0.003} & 0.975 $\pm$ 0.003 \\
KernelSHAP & 0.966 $\pm$ 0.001 & 0.173 $\pm$ 0.002 & 0.014 $\pm$ 0.001 & \textbf{1.000 $\pm$ 0.000} \\
IG & 0.922 $\pm$ 0.006 & 0.261 $\pm$ 0.006 & 0.022 $\pm$ 0.003 & \textbf{1.000 $\pm$ 0.000} \\
\bottomrule
\end{tabular}
\end{table}

For real datasets, ground truth component functions are unavailable. As a diagnostic companion, the supplementary material reports a local conditional pairwise delta agreement analysis, which compares model output deltas with explanation deltas on observed pairs of samples differing primarily in one feature. Results there are mixed: \nobsp{} is strongest on Iris and Breast Cancer, while KernelSHAP attains the strongest agreement on most of the remaining datasets. We report it for transparency rather than as evidence of uniform superiority.

\begin{figure*}[t]
\centering
\includegraphics[width=\textwidth]{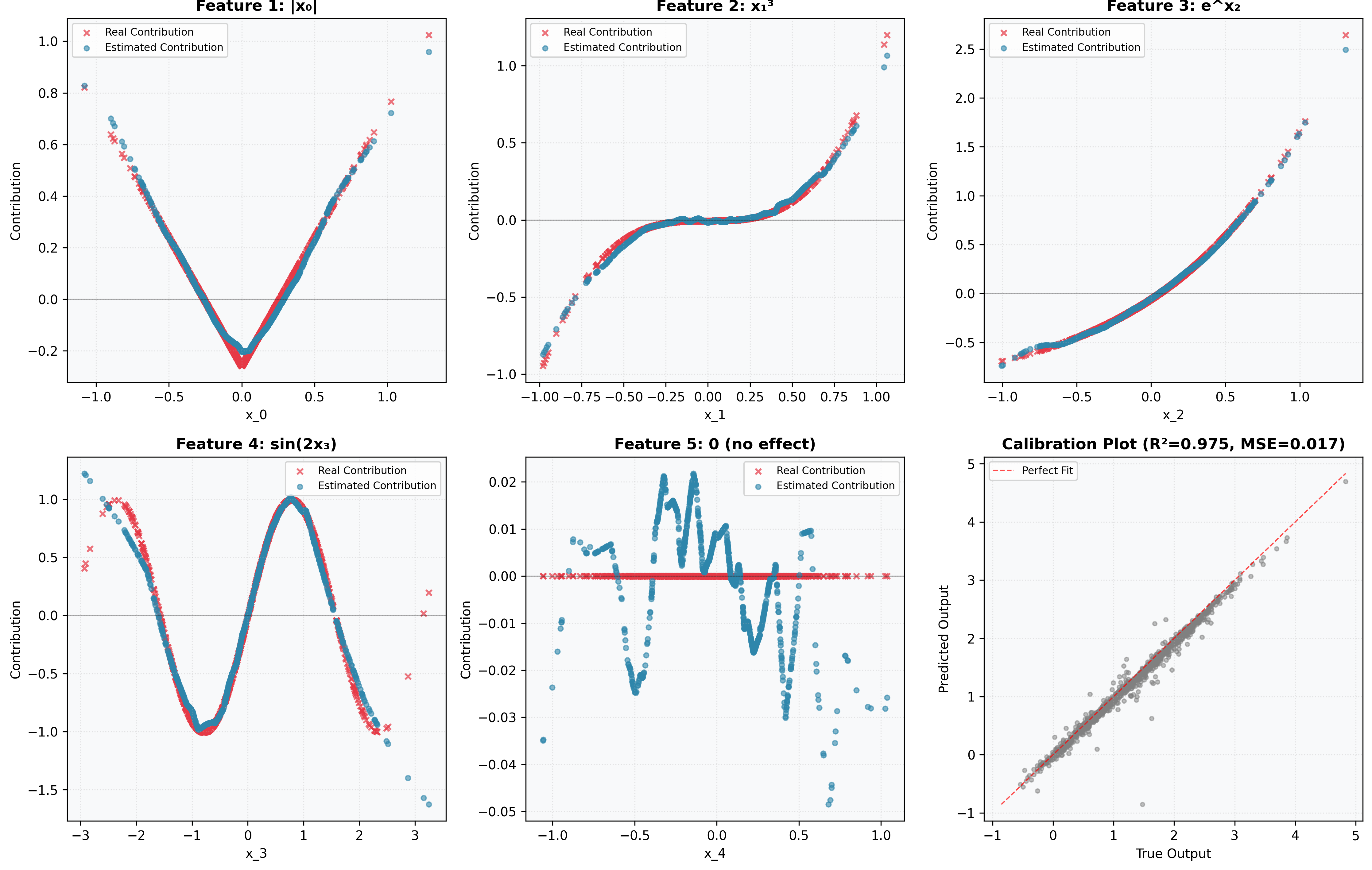}
\caption{Functional decomposition on the synthetic additive regression with known ground truth $y = |x_1| + x_2^3 + e^{x_3} + \sin(2 x_4) + 0 \cdot x_5 + \eta$. Panels show the recovered functional forms $g_k(x_k)$ against the true component functions. \nobsp{} recovers the absolute value, cubic, exponential, and sinusoidal components and correctly assigns the null feature $x_5$ a negligible contribution.}
\label{fig:synthetic_decomposition}
\end{figure*}

\subsection{Vision Results}
\label{sec:results_vision}

\subsubsection{Region Deletion AUC}

Table~\ref{tab:vision_region_deletion} presents Region Deletion AUC across model and dataset combinations, with 200 synchronized evaluation samples per configuration.

\begin{table}[t]
\centering
\caption{Region Deletion AUC on vision benchmarks (higher is better). The metric ranks heatmap patches by attribution, masks the highest ranked regions, and integrates the normalized drop in target class probability; higher values indicate that highlighted regions contain more target class evidence. Values are means over 200 synchronized image samples per row. Bold marks the numerical maximum per row, not statistical significance.}
\label{tab:vision_region_deletion}
\scriptsize
\setlength{\tabcolsep}{2pt}
\begin{tabular}{@{}llcccc@{}}
\toprule
Model & Dataset & GradCAM & GradCAM++ & Eigen-CAM & \nobsp{} \\
\midrule
ResNet-18 & CIFAR-10 & 0.629 & 0.610 & \textbf{0.647} & 0.611 \\
ResNet-18 & TinyImageNet & \textbf{0.718} & 0.699 & 0.632 & 0.697 \\
ResNet-18 & DermaMNIST & \textbf{0.582} & 0.552 & 0.534 & 0.549 \\
MobileNetV2 & CIFAR-10 & 0.601 & 0.604 & 0.580 & \textbf{0.614} \\
MobileNetV2 & TinyImageNet & \textbf{0.659} & \textbf{0.659} & 0.635 & 0.656 \\
MobileNetV2 & DermaMNIST & \textbf{0.553} & 0.550 & 0.533 & 0.535 \\
EfficientNet-B0 & CIFAR-10 & \textbf{0.551} & 0.549 & 0.521 & 0.540 \\
EfficientNet-B0 & TinyImageNet & \textbf{0.720} & 0.709 & 0.659 & 0.703 \\
EfficientNet-B0 & DermaMNIST & 0.561 & \textbf{0.563} & 0.559 & 0.552 \\
RegNetY-400MF & CIFAR-10 & 0.593 & \textbf{0.597} & 0.587 & 0.595 \\
RegNetY-400MF & TinyImageNet & \textbf{0.685} & 0.674 & 0.640 & 0.673 \\
RegNetY-400MF & DermaMNIST & \textbf{0.516} & 0.478 & 0.460 & 0.478 \\
\bottomrule
\end{tabular}
\end{table}

GradCAM is the strongest method overall on this metric (mean 0.614 over the subset, against 0.604 for GradCAM++, 0.600 for \nobsp{}-CAM, and 0.582 for Eigen-CAM). Under this region level metric, \nobsp{}-CAM achieves faithfulness comparable to GradCAM and GradCAM++, with differences that are small relative to variation across datasets and architectures: it attains the row maximum on MobileNetV2 with CIFAR-10 (0.614 vs 0.601 for GradCAM) and is close to GradCAM on the remaining CIFAR-10 rows and on MobileNetV2 with TinyImageNet. GradCAM remains stronger on most TinyImageNet and DermaMNIST rows. Because the same image indices are used for all methods within each configuration, method comparisons are paired; the corresponding bootstrap intervals and significance tests are summarized in Section~\ref{sec:results_stats}. The distinct contribution of \nobsp{}-CAM is therefore not superior deletion performance, but its calibrated decomposition into positive and negative evidence, requiring no backward pass at inference. Results for the full profile of six datasets and four backbones (24 configurations) appear in the supplementary material and show the same pattern.

\subsubsection{Sparseness}

Table~\ref{tab:vision_sparseness} presents sparseness results measuring attribution concentration.

\begin{table}[t]
\centering
\caption{Sparseness on vision benchmarks (higher indicates more focused attributions). Values are means over synchronized subsets of 20--50 images per row. Bold marks the numerical maximum per row, not statistical significance. Model 1 is ResNet-18, model 2 is EfficientNet-B0, and model 3 is RegNetY-400MF.}
\label{tab:vision_sparseness}
\footnotesize
\setlength{\tabcolsep}{2pt}
\begin{tabular}{@{}llcccc@{}}
\toprule
Model & Dataset & GradCAM & GradCAM++ & Eigen-CAM & \nobsp{} \\
\midrule
\multirow{3}{*}{1}
  & CIFAR-10 & 0.258 & 0.259 & \textbf{0.617} & 0.264 \\
  & TinyImageNet & 0.404 & 0.373 & \textbf{0.695} & 0.375 \\
  & DermaMNIST & 0.472 & 0.392 & \textbf{0.718} & 0.402 \\
\midrule
\multirow{3}{*}{2}
  & CIFAR-10 & 0.289 & 0.285 & \textbf{0.529} & 0.293 \\
  & TinyImageNet & 0.599 & 0.652 & \textbf{0.758} & 0.562 \\
  & DermaMNIST & 0.545 & 0.699 & \textbf{0.812} & 0.586 \\
\midrule
\multirow{3}{*}{3}
  & CIFAR-10 & 0.235 & 0.235 & \textbf{0.569} & 0.229 \\
  & TinyImageNet & 0.595 & 0.567 & \textbf{0.745} & 0.564 \\
  & DermaMNIST & 0.624 & 0.544 & \textbf{0.775} & 0.583 \\
\bottomrule
\end{tabular}
\end{table}

Eigen-CAM consistently produces the sparsest attributions, expected from its PCA-based approach concentrating attribution on dominant directions. GradCAM, GradCAM++, and \nobsp{}-CAM produce comparable sparseness, with \nobsp{}-CAM slightly less sparse on average. This reflects \nobsp{}-CAM's inclusion of contributions from multiple channels rather than focusing on a single dominant component. Read together with Table~\ref{tab:vision_region_deletion}, sparseness and faithfulness dissociate: Eigen-CAM is the sparsest method yet obtains the lowest Region Deletion AUC, so concentration of an explanation does not by itself indicate that the highlighted regions carry class evidence.

\subsubsection{Qualitative Analysis}

Figure~\ref{fig:cam_comparison} presents representative CAM comparisons. On general object recognition, all methods highlight similar discriminative regions, with \nobsp{}-CAM producing comparable localization to gradient based approaches.

\begin{figure*}[h!]
\centering
\includegraphics[width=\textwidth]{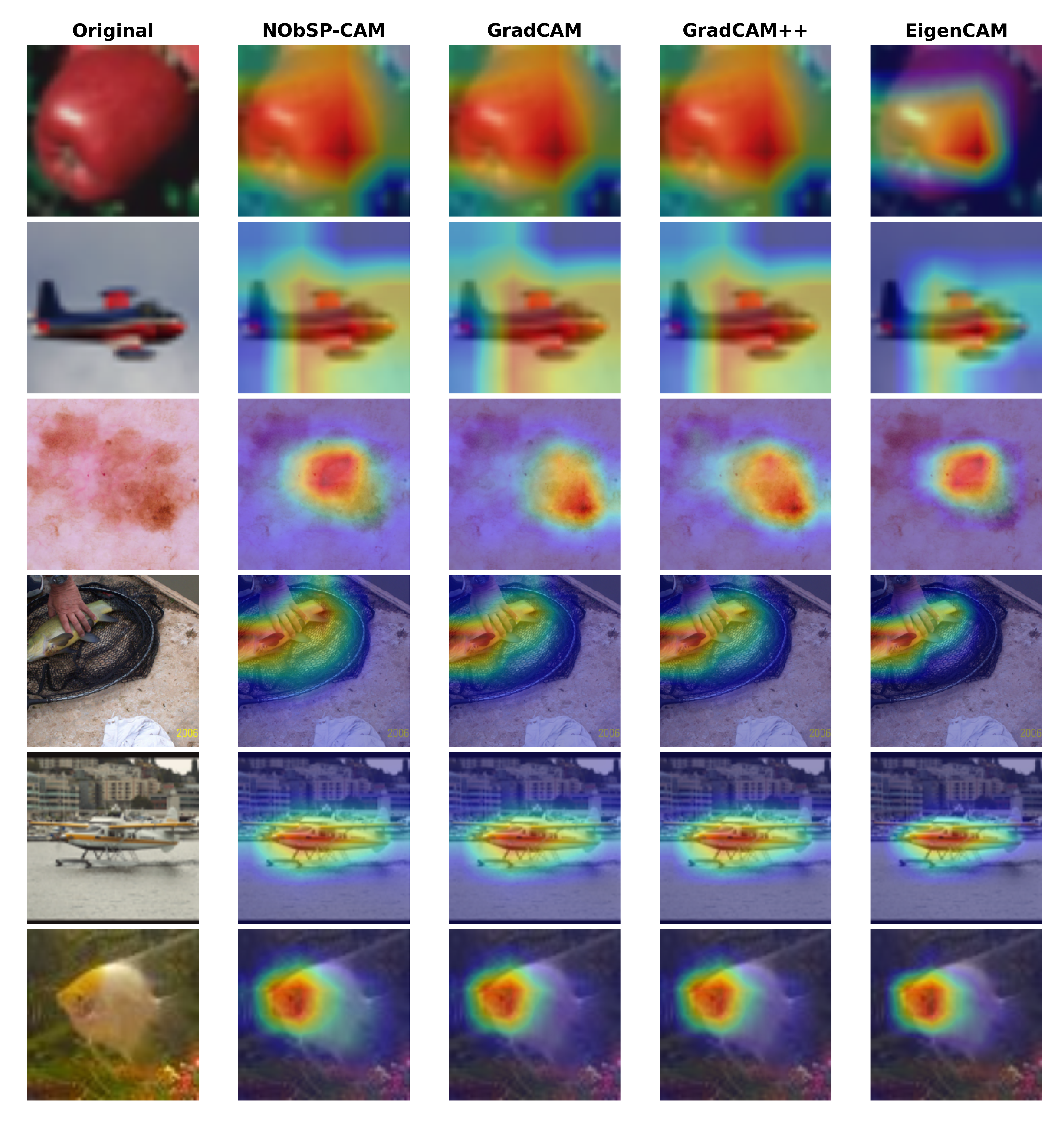}
\caption{Qualitative comparison of class activation maps across datasets (EfficientNet-B0). Each row shows a different sample with columns: original image, \nobsp{}-CAM (positive contributions), GradCAM, GradCAM++, and Eigen-CAM. Row labels indicate predicted class and confidence. All methods highlight similar discriminative regions.}
\label{fig:cam_comparison}
\end{figure*}

On medical imaging (DermaMNIST, Figure~\ref{fig:dermamnist_comparison}), \nobsp{}-CAM qualitatively suggests improved background suppression on these examples, highlighting lesion cores while reducing activation on surrounding healthy skin. A systematic localization evaluation with lesion masks remains future work.

\begin{figure}[t]
\centering
\includegraphics[width=\columnwidth]{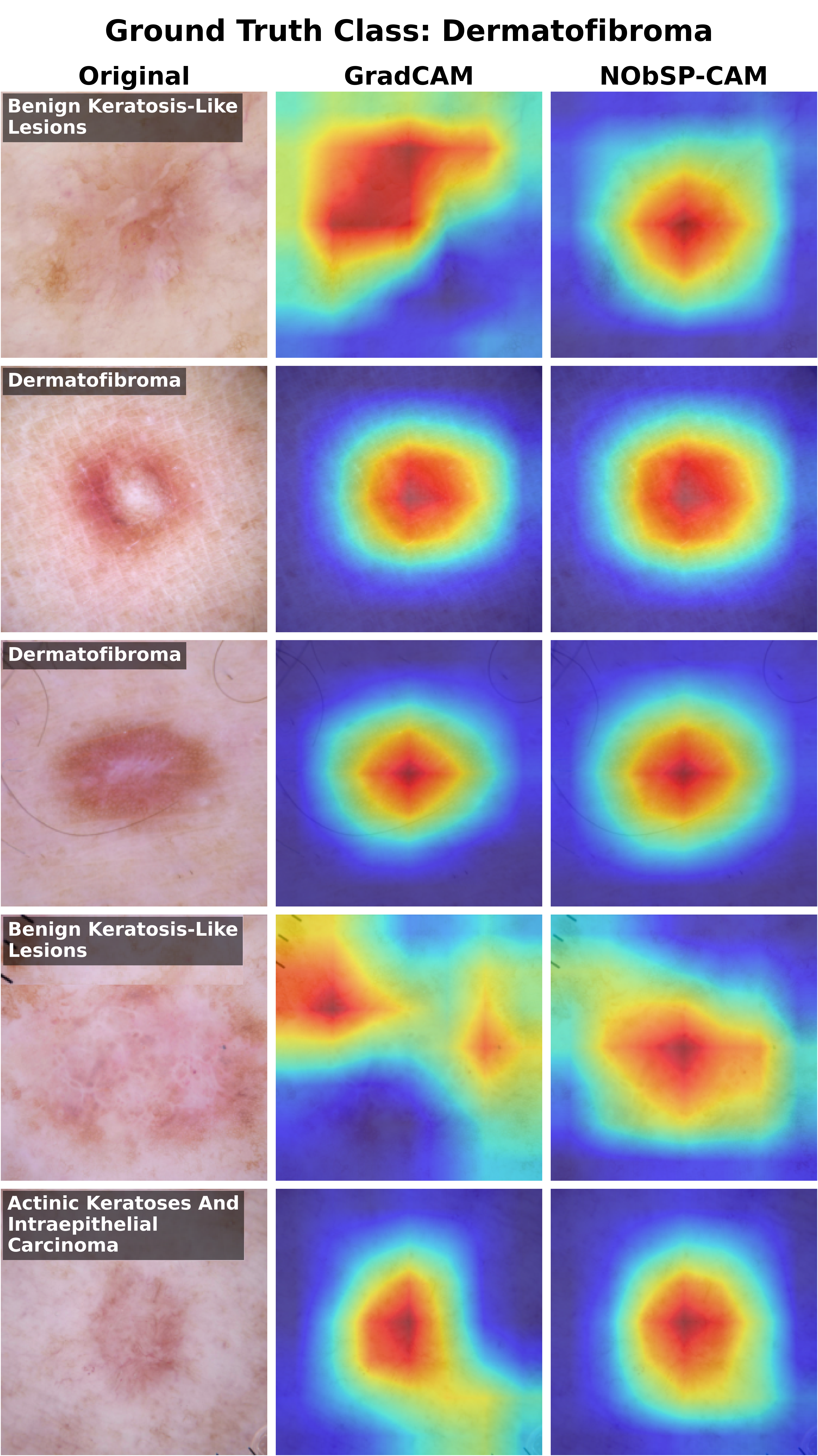}
\caption{CAM comparison on DermaMNIST dermatoscopic images. Columns: original image, GradCAM, \nobsp{}-CAM$^+$ (positive evidence). In these examples, \nobsp{}-CAM shows stronger background suppression, highlighting lesion cores while reducing activation on surrounding skin.}
\label{fig:dermamnist_comparison}
\end{figure}

\subsection{Embedding Analysis Results}
\label{sec:results_embedding}

We evaluate the \nobsp{} embedding space on TinyImageNet (200 classes, 224$\times$224) using ResNet-18, comparing contribution embeddings computed using the model's predicted class against raw penultimate layer activations. The training set (100,000 images) serves as the kNN database and the validation set (10,000 images) as queries, with $K=8$ neighbors and Euclidean distance after z-score followed by L2 normalization.

\subsubsection{Neighborhood Purity}

Table~\ref{tab:purity} presents purity@$K$ results. The \nobsp{} embedding improves mean class purity (0.713 vs 0.654 for raw activations at $K=8$) and reaches a median purity of 1.0, while reducing mean neighbor distance by more than half. This suggests that contribution vectors capture the classifier's decision geometry rather than merely feature space structure.

\begin{table}[t]
\scriptsize
\centering
\caption{Neighborhood retrieval on TinyImageNet ($K=8$). \nobsp{} 
embeddings computed using the predicted class improve purity and 
substantially reduce neighbor distance compared to raw activations.}
\label{tab:purity}
\begin{tabular}{@{}lccc@{}}
\toprule
Embedding & Purity@8 Mean & Purity@8 Median & Mean kNN Dist. \\
\midrule
Raw Activations & 0.654 & 0.875 & 0.855 \\
\nobsp{} (Pred) & \textbf{0.713} & \textbf{1.000} & \textbf{0.408} \\
\bottomrule
\end{tabular}
\end{table}

The \nobsp{} embedding improves mean purity by approximately 6 percentage points (0.654 to 0.713) while reducing mean kNN distance by more than half (0.855 to 0.408). The median purity of 1.0 indicates that the majority of validation queries achieve perfect same class neighborhoods; the lower mean reflects the minority of misclassified samples, whose contribution vectors are computed against the incorrect predicted class and consequently attract neighbors from that wrong class. This asymmetry is itself informative, low purity under \nobsp{} embedding signals likely misclassification, providing an unsupervised confidence measure that complements softmax scores.

Importantly, high neighborhood purity is only interpretable jointly
with prediction correctness. For correctly classified samples, high
purity provides trustworthy visual evidence supporting the
prediction (Figure~\ref{fig:ret_camel},~\ref{fig:ret_poncho}). For
confidently misclassified samples, however, high purity can
reinforce incorrect decisions by retrieving coherent evidence for
the wrong class (Figure~\ref{fig:ret_fence}), while activation
neighbors may preserve useful ambiguity. This asymmetry suggests
that \nobsp{} retrieval is most valuable as an auditing tool when
combined with complementary signals such as activation space
retrieval or softmax calibration.



\subsubsection{Qualitative Retrieval Analysis}

Figure~\ref{fig:retrieval_panels} presents retrieval panels for
representative boundary and outlier cases, comparing \nobsp{} and
activation neighbors side by side. Boundary cases are defined by
activation space purity in $[0.4, 0.7]$, identifying queries that
are ambiguous in the raw feature space; outlier cases are queries
whose class conditional centroid distance exceeds the 95th
percentile.

In panel~(\subref{fig:ret_camel}), an Arabian camel image is
correctly classified with high confidence (0.91). The \nobsp{}
neighbors are four camels at distances $d \leq 0.43$, providing
coherent visual evidence for the prediction. The activation neighbors
include African elephants, large animals in similar outdoor
settings, that share feature space structure but belong to the wrong
class. Panel~(\subref{fig:ret_poncho}) shows an analogous pattern
for a poncho (conf.\ 0.82): \nobsp{} retrieves four ponchos while
activation neighbors include vestments and fur coats, semantically
related garments from incorrect classes. In both cases, \nobsp{}
embedding isolates the decision relevant features, producing
neighbors that directly support the prediction.

Panel~(\subref{fig:ret_fence}) reveals an important limitation. A
picket fence is confidently misclassified as a space heater (0.90),
and \nobsp{} retrieves four space heaters, objects whose repeating
grid patterns visually resemble the fence slats. The \nobsp{}
embedding faithfully reflects the model's reasoning, but in doing so
reinforces an incorrect decision with coherent evidence. The
activation neighbors, by contrast, include two actual picket fences,
exposing an ambiguity that the \nobsp{} embedding suppresses. This
demonstrates that high \nobsp{} purity is only reliable as
supporting evidence when the underlying prediction is correct; for
confident misclassifications, activation space retrieval can provide
a complementary diagnostic.

Panels~(\subref{fig:ret_arch}) and~(\subref{fig:ret_seashore}) show
outlier cases. In~(\subref{fig:ret_arch}), a triumphal arch
misclassified as a suspension bridge (conf.\ 0.54) produces mixed
architectural neighbors in both spaces, bridges, arches, fountains,
viaducts, reflecting genuine visual ambiguity in the stone-arch
query. The low confidence combined with mixed neighborhoods
consistently signals an unreliable prediction.
In~(\subref{fig:ret_seashore}), an image labeled as rugby ball is
predicted as seashore (conf.\ 0.66); both \nobsp{} and activation
neighbors are dominated by beach scenes, consistent with the visual
content. The absence of rugby ball related images in either
embedding space suggests a potential labeling error rather than a
model failure, illustrating how retrieval can serve as an auditing
tool for dataset quality.

Together, these panels show that \nobsp{} retrieval enhances
interpretability by providing class consistent evidence when the
model is correct, while its behavior under misclassification, either
reinforcing errors or exposing uncertainty, provides complementary
diagnostic information about prediction reliability.

\begin{figure*}[!t]
\centering
\begin{subfigure}[t]{0.49\textwidth}
\centering
\includegraphics[width=\textwidth]{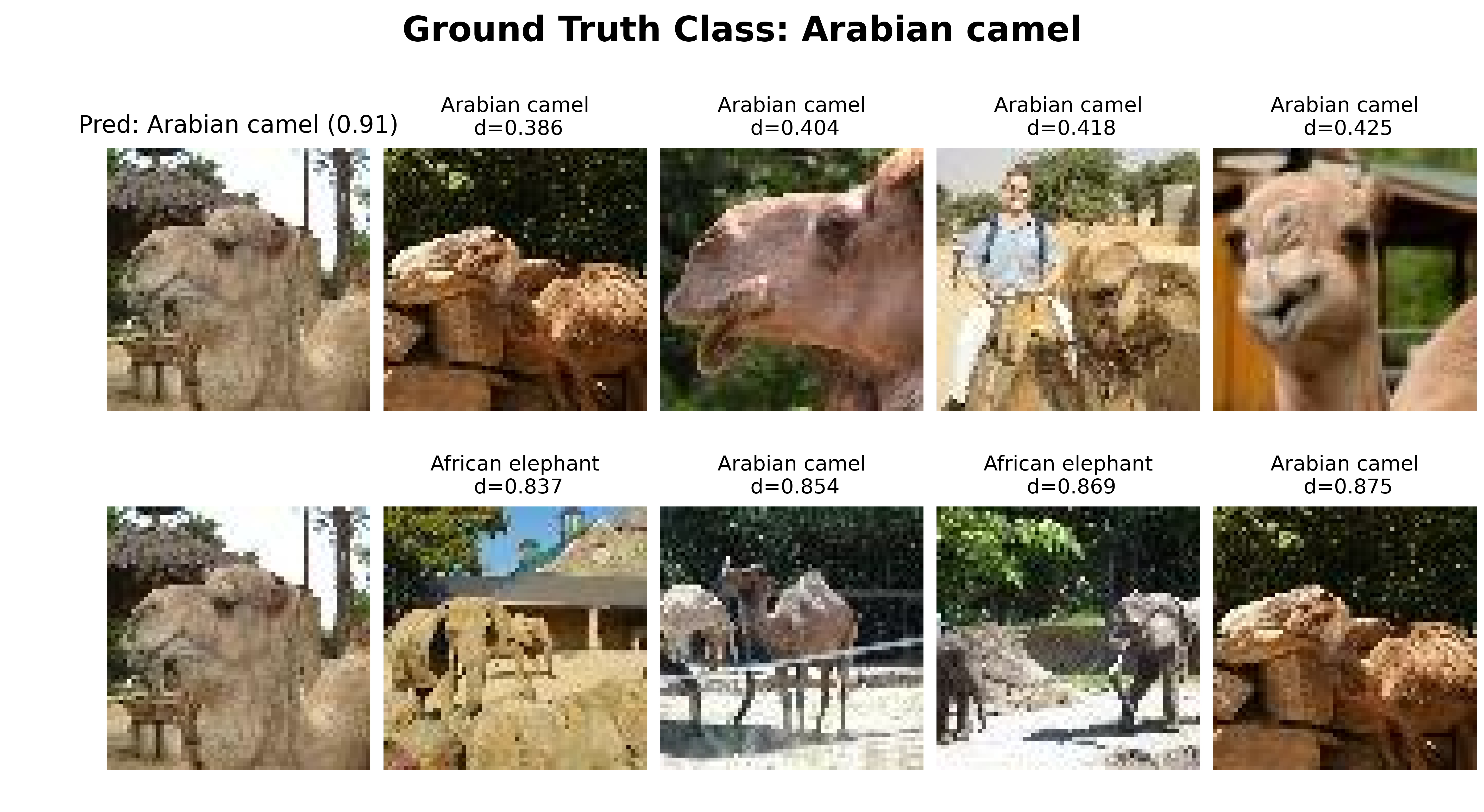}
\caption{}
\label{fig:ret_camel}
\end{subfigure}
\hfill
\begin{subfigure}[t]{0.49\textwidth}
\centering
\includegraphics[width=\textwidth]{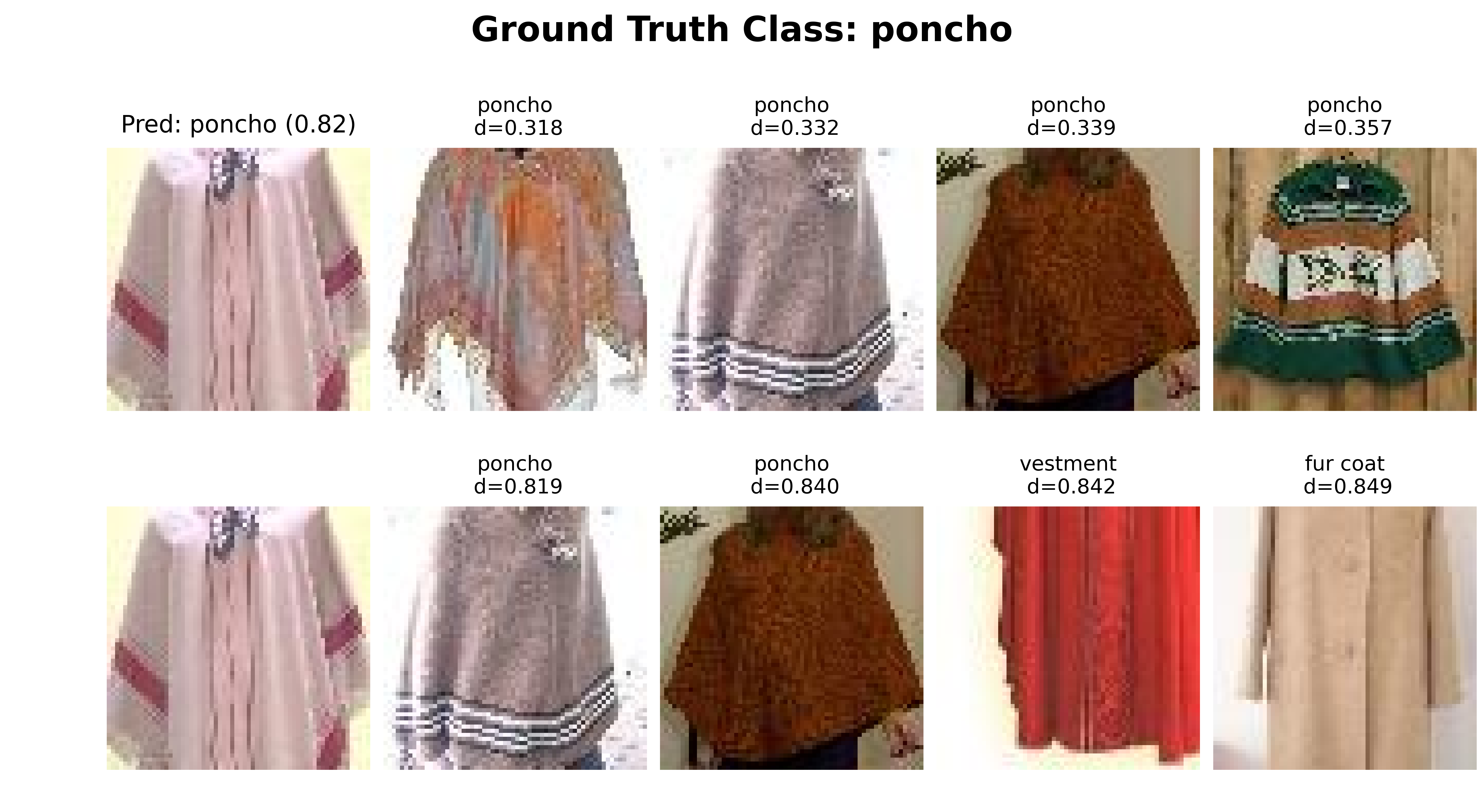}
\caption{}
\label{fig:ret_poncho}
\end{subfigure}

\vspace{3mm}

\begin{subfigure}[t]{0.49\textwidth}
\centering
\includegraphics[width=\textwidth]{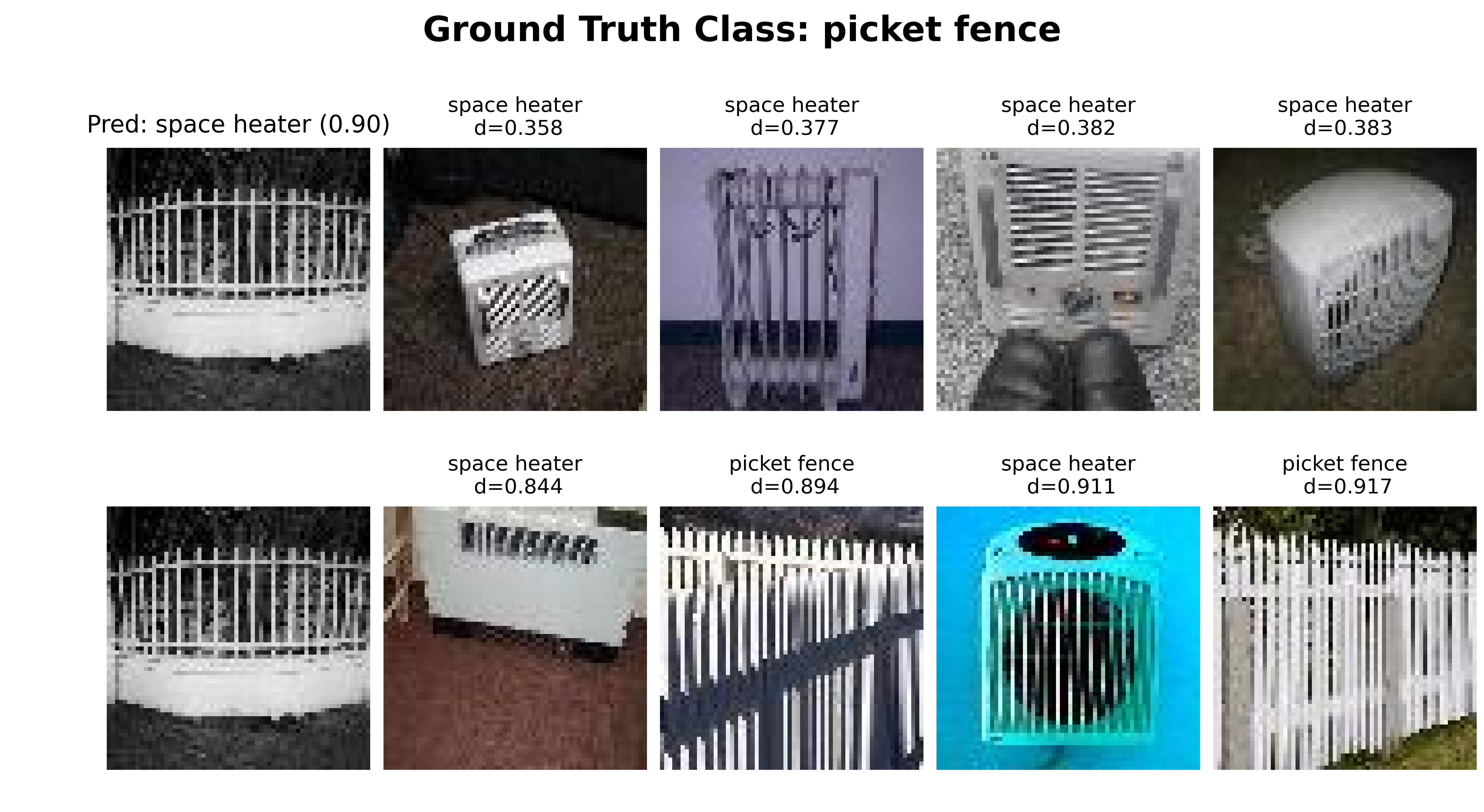}
\caption{}
\label{fig:ret_fence}
\end{subfigure}
\hfill
\begin{subfigure}[t]{0.49\textwidth}
\centering
\includegraphics[width=\textwidth]{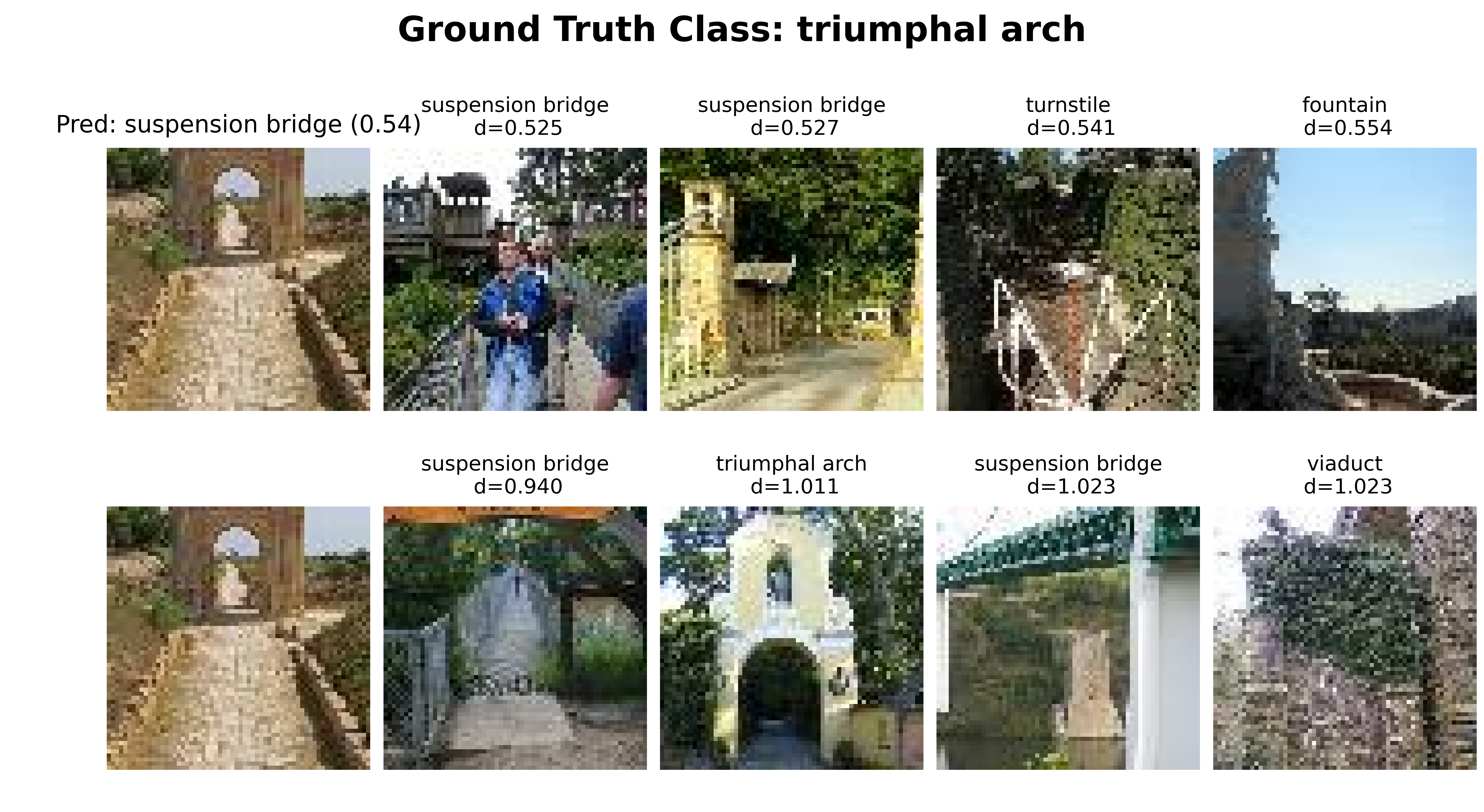}
\caption{}
\label{fig:ret_arch}
\end{subfigure}

\vspace{3mm}

\begin{subfigure}[t]{0.45\textwidth}
\centering
\includegraphics[width=\textwidth]{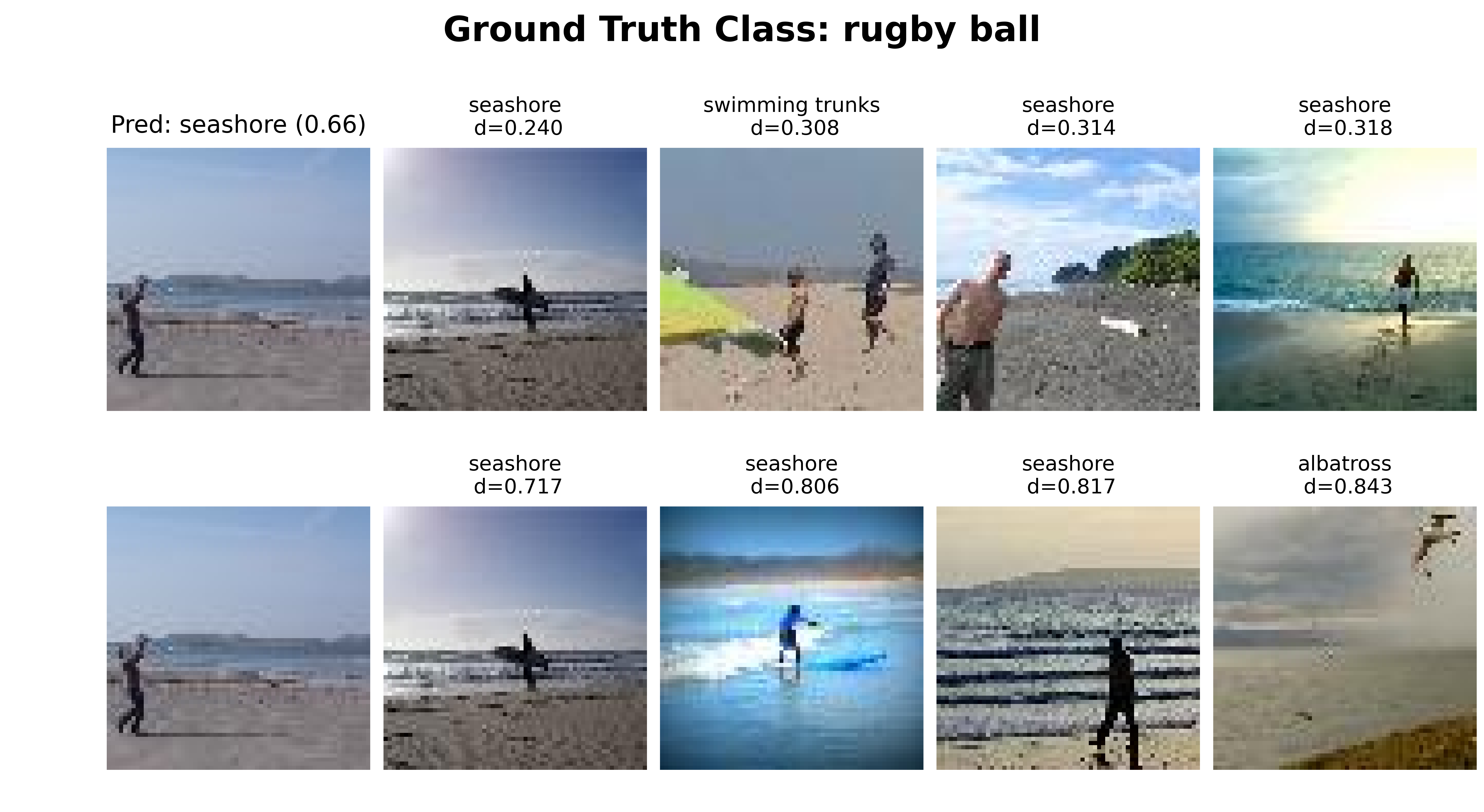}
\caption{}
\label{fig:ret_seashore}
\end{subfigure}

\caption{Similarity-based retrieval on TinyImageNet using \nobsp{}
embeddings (top row in each panel) vs.\ raw activations (bottom row).
Each panel shows a query image (left) with its predicted class and
confidence, followed by its four nearest neighbors labeled with
ground-truth class and Euclidean distance~$d$.
Panels~(a)--(b): boundary cases where \nobsp{} retrieves
class-consistent neighbors while activations return visually similar
but wrong-class images.
Panel~(c): a confident misclassification where \nobsp{} reinforces
the error while activations expose the ambiguity.
Panels~(d)--(e): outlier cases where both spaces reveal genuine
uncertainty, with~(e) suggesting a potential labeling error.}
\label{fig:retrieval_panels}
\end{figure*}

\subsection{Statistical Summary}
\label{sec:results_stats}

Region Deletion AUC comparisons are paired at the level of image and model evaluations, within each dataset and model configuration, all methods are evaluated on identical image indices. We report paired bootstrap confidence intervals (10,000 resamples) and sign flip randomization tests with Holm correction for the \nobsp{}-CAM minus GradCAM difference.

On the paper subset (Table~\ref{tab:vision_region_deletion}), the difference is not significant on any CIFAR-10 row for any of the four backbones (on MobileNetV2 the point estimate favors \nobsp{}-CAM, $+0.013$, 95\% CI $[0.0003, 0.026]$, corrected $p = 1$), nor on MobileNetV2 with TinyImageNet or EfficientNet-B0 with DermaMNIST. GradCAM is significantly ahead on the remaining six TinyImageNet and DermaMNIST rows (corrected $p \leq 0.008$), with the largest difference on RegNetY-400MF with DermaMNIST ($-0.038$, 95\% CI $[-0.051, -0.027]$).

Aggregating the full profile of six datasets and four backbones (24 configurations, 200 synchronized samples each, 4{,}784 paired image and model evaluations), the paired \nobsp{}-CAM minus GradCAM difference is $-0.011$ with 95\% CI $[-0.012, -0.009]$ and Holm corrected sign flip $p = 0.0003$. This small GradCAM advantage is statistically detectable because $n$ is large, but it is minor in absolute Region Deletion AUC units and comparable to typical row to row variation. Sparseness differences are large and consistent in direction: Eigen-CAM is the sparsest method in every configuration, while \nobsp{}-CAM is slightly less sparse than GradCAM, consistent with its richer channel decomposition.

\section{Discussion}
\label{sec:discussion}

\subsection{Why \nobsp{} Works}

A natural question is, why does \nobsp{} recover meaningful functional contributions? The main idea is that the neural network has already learned the nonlinear transformations. When the network is well trained, the penultimate representation $\vect{z} = f_{NN \to L-1}(\vect{x})$ encodes how the model processes each input feature. \nobsp{} does not discover these transformations; it reads them out from the learned representation.

This explains our empirical observation that \nobsp{} recovers true functional forms when the model is well-trained but produces less meaningful results for poorly trained models. This is not a limitation, it is the correct behavior for a method that extracts the learned function rather than imposing external structure. The quality of \nobsp{} decomposition directly reflects the quality of the learned underlying model.

The oblique projection framework is essential when learned representations exhibit overlap. Neural networks commonly represent multiple features in shared dimensions, a phenomenon known as superposition~\cite{elhage2022superposition}. Classical ANOVA decomposition assumes orthogonality and fails under superposition. Oblique projections isolate the unique contribution of each feature by projecting along the space of other features, correctly handling representational overlap.

An important practical finding is that samples with low \nobsp{} neighborhood purity under predicted class targeting reliably correspond to misclassified inputs. This positions the embedding as an unsupervised error-detection tool. At inference time, without any labels, low purity flags predictions that warrant human review or additional verification.

\subsection{Connection to Mechanistic Interpretability}

The emerging field of mechanistic interpretability seeks to reverse engineer neural network computations by identifying interpretable features and circuits~\cite{olah2020zoom,conmy2023towards,wang2022interpretability}, decomposing internal representations with sparse dictionary learning~\cite{cunningham2023sparse,bricken2023monosemanticity}, and interpreting feed-forward layers as key-value memories~\cite{geva2021transformer}. \nobsp{} offers a complementary perspective. While mechanistic interpretability typically operates at the neuron or attention head level, \nobsp{} provides input-output level interpretability through functional decomposition. The subspaces $\mathcal{V}_k$ can be viewed as encoding learned features, while the contribution functions $g_k(x_k)$ describe how these features influence predictions. Extending \nobsp{} to transformer architectures, where the residual stream provides natural subspace structure, remains an interesting direction for future work.

\subsection{Metric Considerations}

No single XAI evaluation metric is method neutral. Deletion based metrics favor perturbation stable scalar attributions and may penalize functional decompositions; completeness and reconstruction metrics favor methods designed to satisfy summation constraints, such as Integrated Gradients and SHAP; CAM metrics depend strongly on heatmap granularity and normalization, as the contrast between pixel level faithfulness correlation and Region Deletion AUC in Section~\ref{sec:results_vision} illustrates. We therefore report conventional metrics for comparability, and evaluate \nobsp{}'s central claim with function level metrics whenever ground truth is available.

This positioning is reflected plainly in the results. \nobsp{} does not dominate scalar attribution metrics. Integrated Gradients attains the best tabular faithfulness, KernelSHAP attains the strongest local pairwise agreement on several real datasets, and GradCAM is slightly ahead on Region Deletion AUC. \nobsp{}'s contribution is a different explanation object, explicit per feature contribution functions with separable positive and negative evidence, whose functional recovery can be tested directly in known function settings; there it is the most accurate of the compared methods (Table~\ref{tab:frs}).

\subsection{Limitations}

Several limitations merit consideration.

\textbf{Calibration requirement.} Unlike gradient methods that operate per-sample, \nobsp{} requires calibration on multiple samples to estimate oblique projections. While this is a one time cost that amortizes over inference, it prevents immediate application to new models without calibration data.

\textbf{Linear final layer assumption.} The theoretical framework assumes a linear mapping from penultimate layer to outputs. While this holds for standard classification and regression heads, architectures with nonlinear output layers would require adaptation.

\textbf{Baseline choice.} Setting inactive features to zero assumes normalized inputs where zero represents the population mean. For non-normalized data or domains where zero has different semantics, alternative baselines may be needed.

\textbf{Higher-order interactions.} Computing interaction terms $g_{km}(x_k, x_m)$ requires $\binom{d}{2}$ additional subspace constructions. For high-dimensional inputs, exhaustive interaction analysis becomes computationally prohibitive, though targeted analysis of specific feature pairs remains tractable. Diagnostics based on the size of the interaction residual proved difficult to interpret in preliminary experiments and are left to future work.

\textbf{Negative contributions.} The \nobsp{} framework naturally produces both positive and negative contribution maps, where negative values indicate evidence against the target class. While potentially informative, interpretation of negative contributions requires further study. The present work focuses exclusively on positive contributions for direct comparability with established methods; systematic analysis of negative evidence maps remains an avenue for future investigation.

\subsection{Retrieval as a Diagnostic Tool}

The embedding analysis reveals a nuanced relationship between
neighborhood purity and prediction reliability. Low \nobsp{} purity
reliably flags uncertain or incorrect predictions, as mixed class
neighbors indicate the query lies near a decision boundary. However,
high purity does not guarantee correctness. When the model is
confidently wrong, the contribution embedding retrieves coherent
evidence for the incorrect class
(Figure~\ref{fig:ret_fence}), potentially masking the error. This
asymmetry means that \nobsp{} purity and activation space purity
provide complementary diagnostics, the former reflects the
classifier's decision geometry while the latter reflects
feature space structure. Developing joint confidence measures that
combine both signals is a promising direction for future work.
Additionally, retrieval can serve as a dataset auditing tool.
When both embedding spaces agree that a sample's visual content is
inconsistent with its label (Figure~\ref{fig:ret_seashore}), this
flags potential annotation errors for human review.

\subsection{Future Directions}

Several extensions merit exploration. First, the \nobsp{} embedding space enables similarity based retrieval that improves class purity over raw activations; leveraging this for active learning, prototype selection, or dataset curation could yield practical benefits, particularly when combined with activation space retrieval for joint confidence estimation. Second, features with consistently small contributions across their range are natural candidates for pruning; investigating the connection between \nobsp{} contributions and network compression is promising. Third, extending \nobsp{} to sequence models and transformers, where attention patterns and residual streams provide rich subspace structure, could broaden applicability to language and multimodal domains.

\section{Conclusion}
\label{sec:conclusion}

We presented \nobsp{}, a mathematically principled framework for neural network interpretability based on oblique subspace projections. The key insight is that all nonlinearity resides in the mapping to the penultimate layer, while the final layer is linear, enabling a projection based decomposition into per feature main effects with an explicit residual capturing interactions.

Our contributions include: (1) a theoretical framework connecting oblique projections to functional ANOVA decomposition, with uniqueness guarantees under explicit assumptions; (2) an efficient algorithm reducing complexity from $O(N^2)$ to $O(d_z)$ for out-of-sample evaluation; (3) \nobsp{}-CAM, extending the framework to convolutional networks with class activation maps grounded in the oblique projection framework; and (4) preliminary evidence that \nobsp{} contribution vectors form a useful retrieval embedding, improving neighborhood purity from 0.65 to 0.71 over raw activations and halving mean neighbor distances, without requiring ground truth labels.

Experiments across tabular and vision benchmarks show that \nobsp{} achieves faithfulness comparable to established attribution methods while providing a complementary form of interpretability unavailable from scalar attributions, explicit per feature contribution functions. Unlike methods designed for either local or global interpretation, \nobsp{} provides unified analysis, the same decomposition yields both instance specific explanations and dataset wide functional relationships.

As neural networks are deployed in high risk domains requiring transparency and accountability, mathematically grounded interpretability methods become essential. \nobsp{} offers a mathematical framework to understanding how learned representations transform inputs into predictions.


\bibliographystyle{IEEEtran}

\begin{IEEEbiography}[{\includegraphics[width=1in,height=1.25in,clip,keepaspectratio]{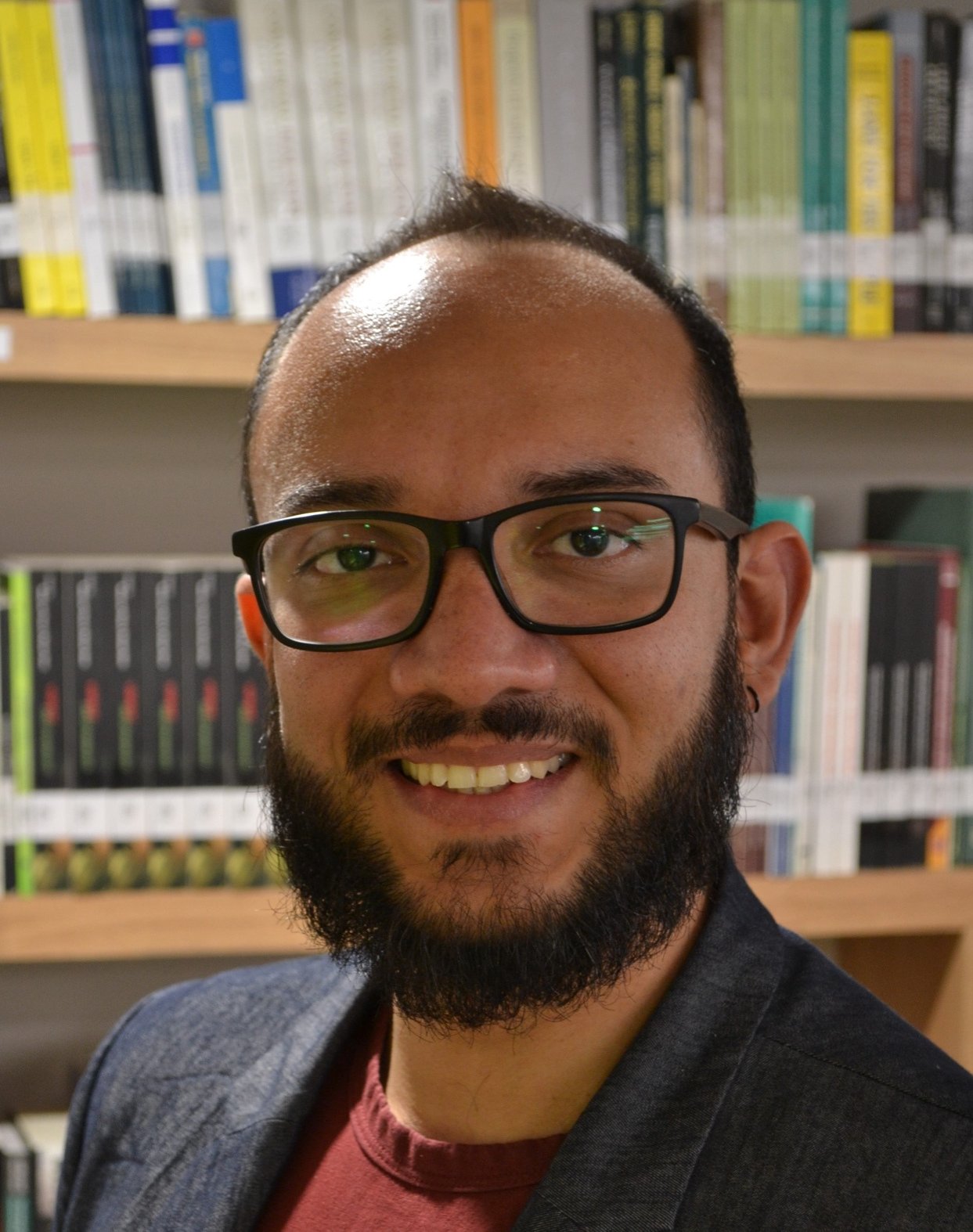}}]{Alexander Caicedo}
Alexander Caicedo received the B.Eng. degree (Hons.) in Electronics Engineering and the M.Sc. degree in Industrial Control Engineering from the University of Ibagué, Colombia. He earned his Ph.D. degree in Engineering from KU Leuven, Belgium. He is currently an Associate Professor with the Department of Electronic Engineering at Pontificia Universidad Javeriana, Bogotá. His research interests lie at the intersection of biomedical signal processing, machine learning, and advanced data analytics. His work is primarily focused on the development of interpretable and explainable artificial intelligence models designed to bridge the gap between complex algorithmic outputs and transparent decision-making in practical settings. Additionally, he is dedicated to the theoretical advancement of machine learning architectures, specifically aiming to optimize training efficiency and reduce computational resource requirements to promote the democratization of high-performance AI. Through his research, he seeks to integrate sophisticated signal processing techniques with robust learning frameworks to enhance diagnostic and prognostic capabilities across diverse application domains.
\end{IEEEbiography}

\begin{IEEEbiography}[{\includegraphics[width=1in,height=1.25in,clip,keepaspectratio]{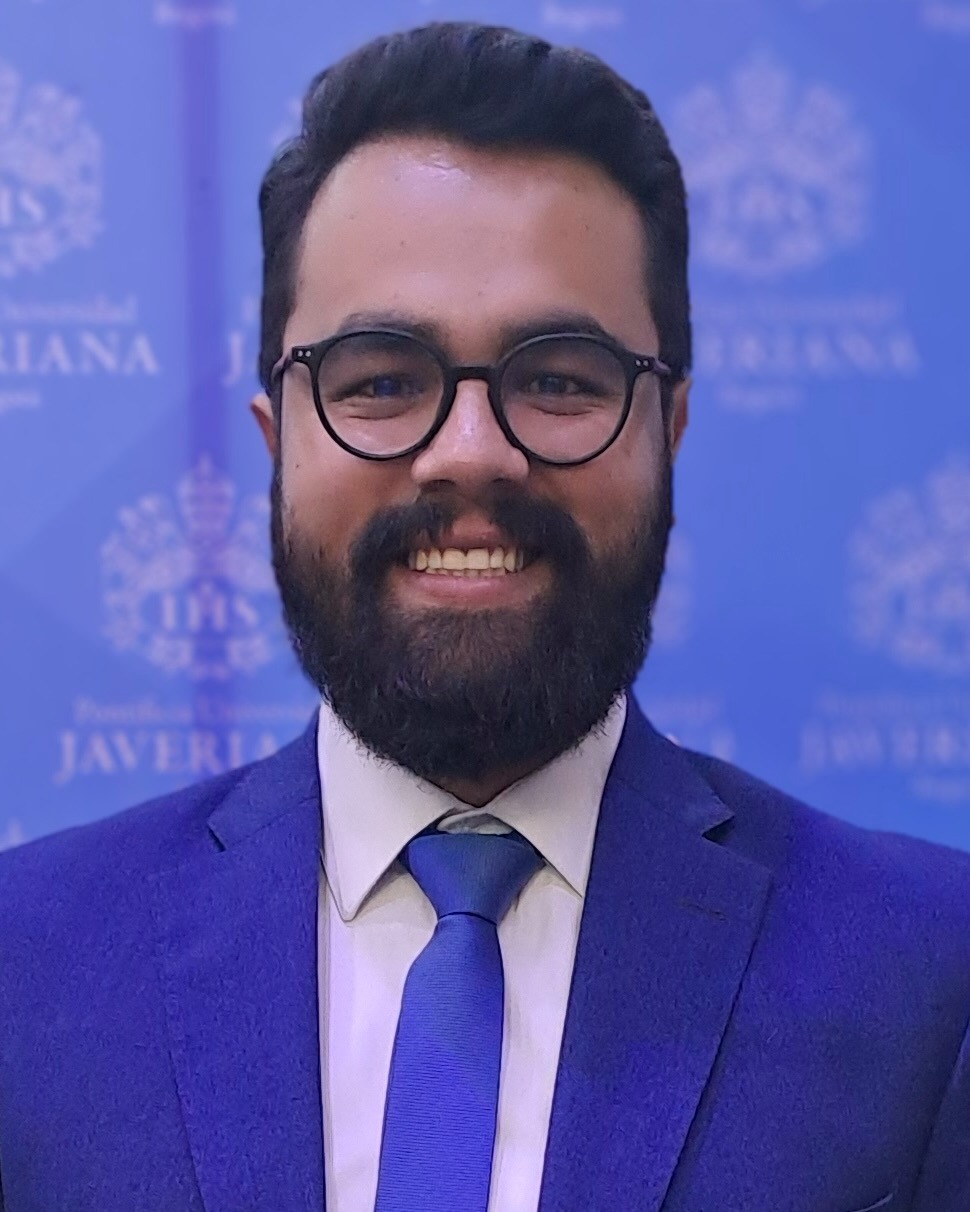}}]{V\'ictor De La Hoz}
received the B.S.\ degree in mathematics from Universidad Nacional de Colombia, Bogot\'a, Colombia, in 2021, and the M.S.\ degree (cum laude) in artificial intelligence from Pontificia Universidad Javeriana, Bogot\'a, Colombia, in 2024.
He is currently a part-time lecturer at Pontificia Universidad Javeriana, where he teaches artificial intelligence fundamentals and mathematics for AI. He also works as a Data Scientist at MercadoLibre, Colombia, focusing on large language model training, evaluation, and deployment. His research interests include continued pre-training 

of large language models, model merging, and interpretability of deep neural networks.
\end{IEEEbiography}

\begin{IEEEbiography}[{\includegraphics[width=1in,height=1.25in,clip,keepaspectratio]{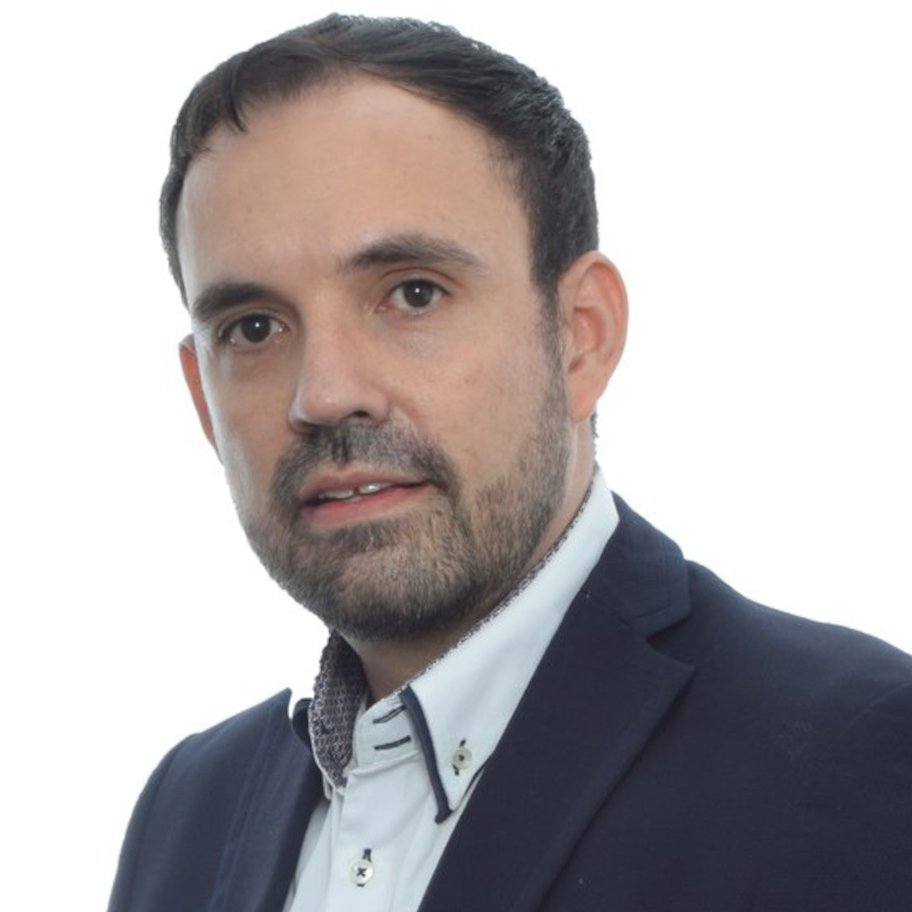}}]{Santiago Alf\'erez}
  holds B.S. degrees in Physics and in Electronics Engineering and an
  M.Sc. in Electronics Engineering from the Universidad Industrial de
  Santander, Colombia, and a Ph.D. in Biomedical Engineering from the
  Universitat Polit\`ecnica de Catalunya (UPC), Barcelona, Spain, where
  he is an Assistant Professor with the Department of Mathematics at the
  East Barcelona School of Engineering (EEBE). His research combines
  computer vision and deep learning for biomedical applications,
  particularly digital pathology and medical imaging. He develops
  explainable AI models for the automatic recognition of cellular
  abnormalities, together with generative modeling and evaluation
  methodologies that make deep learning systems clinically meaningful
  and trustworthy. He is Principal Investigator of a national project on
  explainable deep learning in medical image analysis, and will head the
  UPC team in the recently awarded Marie Sk\l{}odowska-Curie Doctoral
  Network BlueOcean on AI-assisted monitoring of marine ecosystems.
\end{IEEEbiography}


\end{document}


\title{Supplementary Material for\\ NObSP: Functional Decomposition of Neural Networks via Oblique Subspace Projections}

\author{Alexander~Caicedo, V\'ictor~De~La~Hoz, and~Santiago~Alf\'erez}

\markboth{Supplementary Material}{Caicedo \MakeLowercase{\textit{et al.}}: NObSP Supplementary Material}

\maketitle

Numbered references to theorems, algorithms, sections, equations, and citations of the form (3.2), (4.1), and so on refer to the main text.



\section*{Supplementary Material A: Full Proofs}

This supplementary provides complete proofs for the theorems in Section~3. Theorem numbers follow the main text: Theorem~3.1 (Generalized ANOVA), Theorem~3.2 (Prediction Decomposition), Theorem~3.3 (Uniqueness), and Theorem~4.1 (Partial Regression Equivalence).

\subsection*{A.1 Preliminaries and Notation}

Let $\mat{A}_k \in \R^{N \times r_k}$ span subspace $\mathcal{V}_k$, and $\mat{A}_{(k)} \in \R^{N \times r_{(k)}}$ span $\mathcal{V}_{(k)}$, with $\mathcal{V}_k \cap \mathcal{V}_{(k)} = \{\vect{0}\}$.

The orthogonal complement projector is:
\begin{equation}
\mat{Q}_{(k)} = \mat{I} - \mat{A}_{(k)}(\mat{A}_{(k)}^T\mat{A}_{(k)})^{\dagger}\mat{A}_{(k)}^T.
\tag{A.1}
\end{equation}

The oblique projector onto $\mathcal{V}_k$ along $\mathcal{V}_{(k)}$ is:
\begin{equation}
\mat{P}_{k/(k)} = \mat{A}_k (\mat{A}_k^T \mat{Q}_{(k)} \mat{A}_k)^{\dagger} \mat{A}_k^T \mat{Q}_{(k)}.
\tag{A.2}
\end{equation}

\subsection*{A.2 Proof of Theorem~3.1 (Generalized ANOVA)}

\textbf{Theorem~3.1.} \textit{The oblique projection decomposition}
\begin{equation}
\vect{y} = \sum_{k=1}^{d} \mat{P}_{k/(k)}\vect{y} + \vect{r}
\tag{A.3}
\end{equation}
\textit{provides an ANOVA-like decomposition that (i) reduces to classical ANOVA when subspaces are orthogonal, (ii) handles non-orthogonal subspaces, and (iii) provides unique main effects.}

\begin{proof}
We prove each property and characterize $g_k(x_k)$, $g_{k,m}(x_k, x_m)$, and higher-order terms in $\vect{r}$.

\textbf{Property (i): Orthogonal case.}

When $\mathcal{V}_k \perp \mathcal{V}_{(k)}$, any $\vect{v} \in \mathcal{V}_k$ satisfies $\mat{Q}_{(k)} \vect{v} = \vect{v}$. Substituting into (A.2):
\begin{align}
\mat{P}_{k/(k)} &= \mat{A}_k (\mat{A}_k^T \mat{A}_k)^{\dagger} \mat{A}_k^T \notag \\
&= \mat{P}_k^{\perp},
\tag{A.4}
\end{align}
the orthogonal projector used in classical ANOVA.

\textbf{Property (ii): Non-orthogonal case.}

The matrix $\mat{Q}_{(k)}$ first removes the component in $\mathcal{V}_{(k)}$:
\begin{equation}
\vect{y}' = \mat{Q}_{(k)} \vect{y} = \vect{y} - \mat{P}_{(k)} \vect{y}.
\tag{A.5}
\end{equation}

Then $\mat{A}_k (\mat{A}_k^T \mat{Q}_{(k)} \mat{A}_k)^{\dagger} \mat{A}_k^T$ projects onto $\mathcal{V}_k$. For any $\vect{v} \in \mathcal{V}_{(k)}$:
\begin{equation}
\mat{P}_{k/(k)} \vect{v} = \vect{0},
\tag{A.6}
\end{equation}
confirming vectors in the reference subspace are annihilated.

\textbf{Property (iii): Uniqueness.}

The main effect for feature $k$ is:
\begin{equation}
\hat{\vect{y}}_k = \mat{P}_{k/(k)} \vect{y},
\tag{A.7}
\end{equation}
containing $[g_k(x_k^{(1)}), \ldots, g_k(x_k^{(N)})]^T$. Uniqueness follows from uniqueness of the oblique projector given the direct sum decomposition.

\textbf{Functional forms and interactions.}

The full ANOVA decomposition is:
\begin{align}
f(\vect{x}) &= f_0 + \sum_{k=1}^{d} g_k(x_k) \notag \\
&\quad + \sum_{k<m} g_{k,m}(x_k, x_m) + \ldots + G(\vect{x}),
\tag{A.8}
\end{align}
where $g_k$ are main effects, $g_{k,m}$ are second-order interactions, and $G$ captures higher orders.

The oblique projection extracts main effects:
\begin{equation}
\mat{P}_{k/(k)}\vect{y} = \vect{g}_k.
\tag{A.9}
\end{equation}

The residual contains all interaction terms:
\begin{align}
\vect{r} &= \vect{y} - \sum_{k=1}^{d} \mat{P}_{k/(k)}\vect{y} \notag \\
&= \vect{f}_0 + \sum_{k<m} \vect{g}_{k,m} + \text{(higher order)}.
\tag{A.10}
\end{align}

\textbf{Computing second-order interactions.}

For features $k$ and $m$, construct the joint input:
\begin{equation}
\mat{X}_{k,m} = [\vect{0}, \ldots, \vect{x}_k, \ldots, \vect{x}_m, \ldots, \vect{0}],
\tag{A.11}
\end{equation}
with only features $k$ and $m$ active. The representations are:
\begin{equation}
\mat{Z}_{k,m} = f_{NN \to L-1}(\mat{X}_{k,m}).
\tag{A.12}
\end{equation}

The joint projection yields:
\begin{equation}
\mat{P}_{(k,m)/(k,m)}\vect{y} = \vect{g}_k + \vect{g}_m + \vect{g}_{k,m}.
\tag{A.13}
\end{equation}

The pure interaction is obtained by subtraction:
\begin{align}
\vect{g}_{k,m} &= \mat{P}_{(k,m)/(k,m)}\vect{y} \notag \\
&\quad - \mat{P}_{k/(k)}\vect{y} - \mat{P}_{m/(m)}\vect{y}.
\tag{A.14}
\end{align}

\textbf{Higher-order interactions.}

For $r$-th order interactions among $\{k_1, \ldots, k_r\}$:
\begin{align}
\vect{g}_{k_1,\ldots,k_r} &= \mat{P}_{(k_1,\ldots,k_r)}\vect{y} \notag \\
&\quad - \sum_{\text{lower-order terms}}.
\tag{A.15}
\end{align}

The limitation is combinatorial ($\binom{d}{r}$ terms), not mathematical.

\textbf{Complete decomposition.}

Combining the above:
\begin{align}
\vect{y} &= \underbrace{\sum_{k=1}^{d} \vect{g}_k}_{\text{main effects}} + \underbrace{\sum_{k<m} \vect{g}_{k,m}}_{\text{2nd-order}} \notag \\
&\quad + \underbrace{\vect{G} + \vect{f}_0}_{\text{higher-order in } \vect{r}}.
\tag{A.16}
\end{align}

When only main effects are computed:
\begin{equation}
\vect{r} = \vect{f}_0 + \sum_{k<m} \vect{g}_{k,m} + \ldots
\tag{A.17}
\end{equation}
\end{proof}

\subsection*{A.3 Proof of Theorem~3.2 (Prediction Decomposition)}

\textbf{Theorem~3.2.} \textit{For a network with linear final layer:}
\begin{equation}
\hat{y}_j = \sum_{k=1}^{d} g_{k,j}(x_k) + r_j + b_j,
\tag{A.18}
\end{equation}
\textit{where $g_{k,j}(x_k) = \tilde{\vect{z}}_k^{\top} \bm{\Omega}_k \vect{w}_j$ is the contribution of feature $k$ to output $j$, and $\bm{\Omega}_k = \bigl(\tilde{\mat{Z}}_k^T \mat{Q}_{(k)} \tilde{\mat{Z}}_k\bigr)^{\dagger} \tilde{\mat{Z}}_k^T \mat{Q}_{(k)} \tilde{\mat{Z}} \in \R^{d_z \times d_z}$ is the transfer matrix of Theorem~3.2.}

\begin{proof}
The final layer is linear:
\begin{equation}
\hat{y}_j = \vect{w}_j^T \vect{z} + b_j.
\tag{A.19}
\end{equation}

For $N$ samples with $\mat{Z} \in \R^{N \times d_z}$:
\begin{equation}
\hat{\vect{y}}_j = \mat{Z}\vect{w}_j + b_j\vect{1}.
\tag{A.20}
\end{equation}

After centering with $\mat{M} = \mat{I} - \frac{1}{N}\vect{1}\vect{1}^T$:
\begin{equation}
\tilde{\vect{y}}_j = \mat{M}\mat{Z}\vect{w}_j.
\tag{A.21}
\end{equation}

Applying Theorem~3.1:
\begin{equation}
\tilde{\vect{y}}_j = \sum_{k=1}^{d} \mat{P}_{k/(k)} \tilde{\vect{y}}_j + \tilde{\vect{r}}_j.
\tag{A.22}
\end{equation}

The contribution of feature $k$ to output $j$:
\begin{equation}
\vect{g}_{k,j} = \mat{P}_{k/(k)} \tilde{\vect{y}}_j.
\tag{A.23}
\end{equation}

By the partial regression equivalence (Theorem~4.1), the contribution vector factorizes through the isolated representations, $\vect{g}_{k,j} = \tilde{\mat{Z}}_k \bm{\Omega}_k \vect{w}_j$. Its $i$-th entry gives the per sample contribution:
\begin{equation}
g_{k,j}(x_k^{(i)}) = [\vect{g}_{k,j}]_i = \tilde{\vect{z}}_k^{(i)\top} \bm{\Omega}_k \vect{w}_j.
\tag{A.24}
\end{equation}
The dependence on sample $i$ enters only through the isolated representation $\tilde{\vect{z}}_k^{(i)}$, obtained by a forward pass on the isolated input. Note that the projector $\mat{P}_{k/(k)} \in \R^{N \times N}$ acts on vectors in sample space and is never applied to a single latent vector $\vect{z} \in \R^{d_z}$.

The residual contains interactions:
\begin{align}
r_j &= \sum_{k<m} g_{k,m,j}(x_k, x_m) \notag \\
&\quad + \text{(higher-order)} + (b_j - \bar{y}_j).
\tag{A.25}
\end{align}
\end{proof}

\subsection*{A.4 Proof of Theorem~3.3 (Uniqueness)}

\textbf{Theorem~3.3.} \textit{Given trained network $f_{NN}$, data $\mat{X}$, baseline choice, and $\mathcal{V}_k \cap \mathcal{V}_{(k)} = \{\vect{0}\}$ for all $k$, the \nobsp{} decomposition is unique.}

\begin{proof}
\textbf{Step 1: Unique subspaces.}

The matrices $\mat{Z}_k = f_{NN \to L-1}(\mat{X}_k)$ are uniquely determined by the fixed network and data. Thus $\mathcal{V}_k = \text{col}(\tilde{\mat{Z}}_k) \subseteq \R^N$ is unique.

\textbf{Step 2: Unique projectors.}

Given $\mathcal{V}_k$ and $\mathcal{V}_{(k)}$ with $\mathcal{V}_k \cap \mathcal{V}_{(k)} = \{\vect{0}\}$, the oblique projector $\mat{P}_{k/(k)}$ is unique. When subspaces intersect, the pseudoinverse provides the unique minimum-norm solution.

\textbf{Step 3: Unique contributions.}

With unique $\mat{P}_{k/(k)}$ and fixed weights $\vect{w}_j$, the coefficient vector $\bm{\beta}_{k,j} = (\tilde{\mat{Z}}_k^T \mat{Q}_{(k)} \tilde{\mat{Z}}_k)^{\dagger} \tilde{\mat{Z}}_k^T \mat{Q}_{(k)} \tilde{\vect{y}}_j$ is uniquely determined, and thus the per-sample contribution
\begin{equation}
g_{k,j}(x_k^{(i)}) = \tilde{\vect{z}}_k^{(i)\top} \bm{\beta}_{k,j}
\tag{A.26}
\end{equation}
is uniquely determined.

\textbf{Remark.} Different methods (SHAP, LIME, \nobsp{}) yield different attributions because they define ``contribution'' differently, a methodological choice, not non-uniqueness within \nobsp{}.
\end{proof}

\subsection*{A.5 Proof of Theorem~4.1 (Equivalence of Partial Regression Formulation)}

\textbf{Theorem~4.1.} \textit{The contribution $\tilde{\mat{Z}}_k \bm{\beta}_{k,j}$ computed via Algorithm~1 coincides with the oblique projection of $\hat{\vect{y}}_j$ onto $\mathcal{V}_k$ along $\mathcal{V}_{(k)}$: $\tilde{\mat{Z}}_k \bm{\beta}_{k,j} = \mat{P}_{k/(k)} \hat{\vect{y}}_j$.}

\begin{proof}
Consider Algorithm~1 in the limit $\lambda \to 0$, where the regularized inverses converge to Moore--Penrose pseudoinverses.

\textbf{Step 1: The residual steps implement $\mat{Q}_{(k)}$.}
With $\mat{A} = (\tilde{\mat{Z}}_{(k)}^T \tilde{\mat{Z}}_{(k)})^{\dagger} \tilde{\mat{Z}}_{(k)}^T$, Steps 2 and 3 compute
\begin{align}
\hat{\vect{y}}_{res} &= \bigl(\mat{I} - \tilde{\mat{Z}}_{(k)} (\tilde{\mat{Z}}_{(k)}^T \tilde{\mat{Z}}_{(k)})^{\dagger} \tilde{\mat{Z}}_{(k)}^T\bigr) \hat{\vect{y}}_j = \mat{Q}_{(k)} \hat{\vect{y}}_j,
\tag{A.27} \\
\tilde{\mat{Z}}_{res} &= \mat{Q}_{(k)} \tilde{\mat{Z}}_k,
\tag{A.28}
\end{align}
where $\mat{Q}_{(k)}$ is the orthogonal complement projector of (A.1) with $\mat{A}_{(k)} = \tilde{\mat{Z}}_{(k)}$.

\textbf{Step 2: Coefficient formula.}
Step 4 solves the least squares problem on the residuals:
\begin{align}
\bm{\beta}_{k,j}
&= (\tilde{\mat{Z}}_{res}^T \tilde{\mat{Z}}_{res})^{\dagger} \tilde{\mat{Z}}_{res}^T \hat{\vect{y}}_{res} \notag \\
&= \bigl(\tilde{\mat{Z}}_k^T \mat{Q}_{(k)}^T \mat{Q}_{(k)} \tilde{\mat{Z}}_k\bigr)^{\dagger} \tilde{\mat{Z}}_k^T \mat{Q}_{(k)}^T \mat{Q}_{(k)} \hat{\vect{y}}_j.
\tag{A.29}
\end{align}
Since $\mat{Q}_{(k)}$ is an orthogonal projector, it is symmetric ($\mat{Q}_{(k)}^T = \mat{Q}_{(k)}$) and idempotent ($\mat{Q}_{(k)}^2 = \mat{Q}_{(k)}$), hence
\begin{equation}
\bm{\beta}_{k,j} = \bigl(\tilde{\mat{Z}}_k^T \mat{Q}_{(k)} \tilde{\mat{Z}}_k\bigr)^{\dagger} \tilde{\mat{Z}}_k^T \mat{Q}_{(k)} \hat{\vect{y}}_j.
\tag{A.30}
\end{equation}

\textbf{Step 3: Identification with the oblique projection.}
Multiplying (A.30) by $\tilde{\mat{Z}}_k$ and comparing with the projector definition (A.2) evaluated at $\mat{A}_k = \tilde{\mat{Z}}_k$:
\begin{equation}
\tilde{\mat{Z}}_k \bm{\beta}_{k,j}
= \tilde{\mat{Z}}_k \bigl(\tilde{\mat{Z}}_k^T \mat{Q}_{(k)} \tilde{\mat{Z}}_k\bigr)^{\dagger} \tilde{\mat{Z}}_k^T \mat{Q}_{(k)} \hat{\vect{y}}_j
= \mat{P}_{k/(k)} \hat{\vect{y}}_j.
\tag{A.31}
\end{equation}

\textbf{Step 4: Transfer matrix form.}
Applying (A.30) to the centered predictions $\tilde{\vect{y}}_j = \tilde{\mat{Z}} \vect{w}_j$ gives $\bm{\beta}_{k,j} = \bm{\Omega}_k \vect{w}_j$ with $\bm{\Omega}_k = (\tilde{\mat{Z}}_k^T \mat{Q}_{(k)} \tilde{\mat{Z}}_k)^{\dagger} \tilde{\mat{Z}}_k^T \mat{Q}_{(k)} \tilde{\mat{Z}}$, the transfer matrix of Theorem~3.2. We emphasize that $\bm{\Omega}_k$ maps final layer weights to contribution coefficients; it is in general neither symmetric nor idempotent, and it is not a projector on the latent space.

For $\lambda > 0$, Algorithm~1 returns a ridge regularized version of (A.30); the equivalence holds exactly as $\lambda \to 0$.
\end{proof}


\section*{Supplementary B: Algorithm Details}

This supplementary section provides complete algorithmic details for the \nobsp{} framework, including the direct oblique projection method and the detailed subspace construction procedure. These materials complement the efficient partial regression formulation presented in the main text.

\subsection*{B.1 Detailed Subspace Construction}

The construction of feature subspaces proceeds through five steps. We assume inputs have been normalized to zero mean and unit variance; for non-normalized data, replace zeros with feature means throughout.

\textbf{Step 1: Input Normalization.} If not already performed during training, normalize inputs to zero mean and unit variance:
\begin{equation}
x_k^{(i)} \leftarrow \frac{x_k^{(i)} - \bar{x}_k}{\sigma_k},
\tag{B.1}
\end{equation}
where $\bar{x}_k = \frac{1}{N}\sum_{i=1}^{N} x_k^{(i)}$ and $\sigma_k^2 = \frac{1}{N}\sum_{i=1}^{N}(x_k^{(i)} - \bar{x}_k)^2$. This ensures consistency between the zero baseline and the training distribution.

\textbf{Step 2: Isolated Input Construction.} For each feature $k \in \{1, \ldots, d\}$, construct the isolated input matrix where only feature $k$ is active:
\begin{equation}
\mat{X}_k = [\vect{0}, \ldots, \vect{0}, \vect{x}_k, \vect{0}, \ldots, \vect{0}] \in \R^{N \times d},
\tag{B.2}
\end{equation}
where only the $k$-th column contains the observed feature values $\vect{x}_k = [x_k^{(1)}, \ldots, x_k^{(N)}]^T$ and all remaining columns are zero. Explicitly, the $(i, j)$-th entry is:
\begin{equation}
[\mat{X}_k]_{ij} = \begin{cases} x_k^{(i)} & \text{if } j = k \\ 0 & \text{otherwise} \end{cases}.
\tag{B.3}
\end{equation}

\textbf{Step 3: Reference Input Construction.} Define the reference input matrix containing all features except $k$:
\begin{equation}
\mat{X}_{(k)} = [\vect{x}_1, \ldots, \vect{x}_{k-1}, \vect{0}, \vect{x}_{k+1}, \ldots, \vect{x}_d] \in \R^{N \times d}.
\tag{B.4}
\end{equation}
Note that $\mat{X}_k + \mat{X}_{(k)} = \mat{X}$, the original data matrix.

\textbf{Step 4: Representation Extraction.} Compute the penultimate layer representations by passing the isolated and reference inputs through the trained network:
\begin{align}
\mat{Z}_k &= f_{NN \to L-1}(\mat{X}_k) \in \R^{N \times d_z}, \tag{B.5}\\
\mat{Z}_{(k)} &= f_{NN \to L-1}(\mat{X}_{(k)}) \in \R^{N \times d_z}, \tag{B.6}
\end{align}
where $f_{NN \to L-1}: \R^d \to \R^{d_z}$ denotes the network mapping from input to the penultimate layer of dimension $d_z$. Each row of $\mat{Z}_k$ contains the $d_z$-dimensional representation of the corresponding isolated input.

\textbf{Step 5: Mean Centering.} Apply mean-centering to the representation matrices:
\begin{equation}
\tilde{\mat{Z}}_k = \mat{M}\mat{Z}_k, \quad \tilde{\mat{Z}}_{(k)} = \mat{M}\mat{Z}_{(k)},
\tag{B.7}
\end{equation}
where the centering matrix is
\begin{equation}
\mat{M} = \mat{I}_N - \frac{1}{N}\vect{1}_N\vect{1}_N^T \in \R^{N \times N},
\tag{B.8}
\end{equation}
with $\mat{I}_N$ the identity matrix and $\vect{1}_N$ a vector of ones. Centering ensures that the subspaces pass through the origin, required for the oblique projection to have the correct interpretation as deviation from the mean.

The column spaces of $\tilde{\mat{Z}}_k$ and $\tilde{\mat{Z}}_{(k)}$ define the feature subspace $\mathcal{V}_k = \text{col}(\tilde{\mat{Z}}_k) \subseteq \R^N$ and reference subspace $\mathcal{V}_{(k)} = \text{col}(\tilde{\mat{Z}}_{(k)}) \subseteq \R^N$, respectively.

\subsection*{B.2 Direct Oblique Projection Method}

Given the centered subspace bases $\tilde{\mat{Z}}_k$ and $\tilde{\mat{Z}}_{(k)}$ constructed above, the contribution of feature $k$ to output $j$ can be computed via direct oblique projection. Throughout, $\dagger$ denotes the Moore-Penrose pseudoinverse.

\textbf{Orthogonal Complement Projector.} First, compute the orthogonal projector onto the complement of the reference subspace $\mathcal{V}_{(k)}$:
\begin{equation}
\mat{Q}_{(k)} = \mat{I}_N - \tilde{\mat{Z}}_{(k)} \left( \tilde{\mat{Z}}_{(k)}^T \tilde{\mat{Z}}_{(k)} \right)^{\dagger} \tilde{\mat{Z}}_{(k)}^T.
\tag{B.9}
\end{equation}
The matrix $\mat{Q}_{(k)}$ projects any vector onto the orthogonal complement of $\mathcal{V}_{(k)}$, effectively removing components explainable by features other than $k$.

\textbf{Oblique Projector.} The oblique projector onto $\mathcal{V}_k$ along $\mathcal{V}_{(k)}$ is:
\begin{equation}
\mat{P}_{k/(k)} = \tilde{\mat{Z}}_k \mat{G}_k^{\dagger} \tilde{\mat{Z}}_k^T \mat{Q}_{(k)},
\tag{B.10}
\end{equation}
where $\mat{G}_k = \tilde{\mat{Z}}_k^T \mat{Q}_{(k)} \tilde{\mat{Z}}_k$ is the Gram matrix of the target basis projected onto the complement of the reference subspace. This projector satisfies three key properties:
\begin{enumerate}
\item \textit{Range}: $\mat{P}_{k/(k)} \vect{v} \in \mathcal{V}_k$ for all $\vect{v} \in \R^N$
\item \textit{Identity on target}: $\mat{P}_{k/(k)} \vect{v} = \vect{v}$ for all $\vect{v} \in \mathcal{V}_k$
\item \textit{Null on reference}: $\mat{P}_{k/(k)} \vect{v} = \vect{0}$ for all $\vect{v} \in \mathcal{V}_{(k)}$
\end{enumerate}

\textbf{Contribution Computation.} The contribution vector for feature $k$ to output $j$ is:
\begin{equation}
\hat{\vect{y}}_{k,j} = \mat{P}_{k/(k)} \hat{\vect{y}}_j,
\tag{B.11}
\end{equation}
where $\hat{\vect{y}}_j = [\hat{y}_j^{(1)}, \ldots, \hat{y}_j^{(N)}]^T \in \R^N$ contains the predictions for output $j$ across all $N$ samples. Each entry $[\hat{\vect{y}}_{k,j}]_i$ represents the contribution of feature $k$ to the prediction for sample $i$.

Algorithm~\ref{alg:direct_supp} summarizes the complete direct projection procedure.

\begin{algorithm}[H]
\caption{Direct Oblique Projection for \nobsp{}}
\label{alg:direct_supp}
\begin{algorithmic}[1]
\REQUIRE Trained network $f_{NN}$, dataset $\mat{X} \in \R^{N \times d}$, predictions $\hat{\vect{y}}_j \in \R^N$
\ENSURE Contribution vectors $\{\hat{\vect{y}}_{k,j}\}_{k=1}^{d}$
\STATE Normalize $\mat{X}$ to zero mean, unit variance
\FOR{$k = 1$ to $d$}
    \STATE Construct $\mat{X}_k$ (isolated) and $\mat{X}_{(k)}$ (reference)
    \STATE $\mat{Z}_k \gets f_{NN \to L-1}(\mat{X}_k)$
    \STATE $\mat{Z}_{(k)} \gets f_{NN \to L-1}(\mat{X}_{(k)})$
    \STATE $\tilde{\mat{Z}}_k \gets \mat{M}\mat{Z}_k$; \quad $\tilde{\mat{Z}}_{(k)} \gets \mat{M}\mat{Z}_{(k)}$
    \STATE $\mat{Q}_{(k)} \gets \mat{I} - \tilde{\mat{Z}}_{(k)} ( \tilde{\mat{Z}}_{(k)}^T \tilde{\mat{Z}}_{(k)} )^{\dagger} \tilde{\mat{Z}}_{(k)}^T$
    \STATE $\mat{G}_k \gets \tilde{\mat{Z}}_k^T \mat{Q}_{(k)} \tilde{\mat{Z}}_k$
    \STATE $\mat{P}_{k/(k)} \gets \tilde{\mat{Z}}_k \mat{G}_k^{\dagger} \tilde{\mat{Z}}_k^T \mat{Q}_{(k)}$
    \STATE $\hat{\vect{y}}_{k,j} \gets \mat{P}_{k/(k)} \hat{\vect{y}}_j$
\ENDFOR
\RETURN $\{\hat{\vect{y}}_{k,j}\}_{k=1}^{d}$
\end{algorithmic}
\end{algorithm}

\subsection*{B.3 Computational Complexity Analysis}

The direct method has the following costs per feature $k$:

\begin{itemize}
\item \textbf{Representation extraction}: $O(T_{NN})$ for $N$-sample forward pass
\item \textbf{Centering}: $O(N d_z)$
\item \textbf{Computing $\mat{Q}_{(k)}$}: $O(N d_z^2 + d_z^3)$
\item \textbf{Computing $\mat{P}_{k/(k)}$}: $O(N^2 d_z)$ --- \textit{bottleneck}
\item \textbf{Contribution}: $O(N^2)$ matrix-vector product
\end{itemize}

The dominant costs are $O(N^2 d_z)$ computation and $O(N^2)$ storage for $\mat{P}_{k/(k)}$. For $d$ features, total complexity is $O(d N^2 d_z)$.

This motivates the efficient partial regression formulation (Algorithm~1 of the main text), which achieves $O(N d_z^2 + d_z^3)$ computation and $O(d_z)$ storage per feature by avoiding explicit construction of $N \times N$ projection matrices.

\subsection*{B.4 Numerical Considerations}

Several practical considerations arise in implementation:

\textbf{Regularization.} When subspace bases are rank-deficient or ill-conditioned, the pseudoinverse can be unstable. Tikhonov regularization improves stability:
\begin{equation}
\left( \mat{A}^T \mat{A} \right)^{\dagger} \approx \left( \mat{A}^T \mat{A} + \lambda \mat{I} \right)^{-1}
\tag{B.12}
\end{equation}
with small $\lambda > 0$ (typically $10^{-4}$ to $10^{-6}$). The efficient partial regression method in the main text incorporates this regularization directly.

\textbf{Subspace Overlap.} When $\mathcal{V}_k \cap \mathcal{V}_{(k)} \neq \{\vect{0}\}$, the oblique projector is not uniquely defined on the intersection. The pseudoinverse provides the minimum-norm solution, projecting onto the portion of $\mathcal{V}_k$ linearly independent from $\mathcal{V}_{(k)}$.

\textbf{High-Dimensional Inputs.} For $d \gg 100$ features, constructing all isolated inputs simultaneously is memory-intensive. Batch processing is recommended: construct and process features in groups.

\textbf{Sparse Activations.} ReLU networks may produce many zeros in $\mat{Z}_k$, particularly for isolated inputs. This sparsity can be exploited computationally but may cause rank deficiency in subspace bases.


\section*{Supplementary Material C: Extended Vision Results}

\subsection*{C.1 Region Deletion AUC on the Full Profile}

Table~\ref{tab:vision_region_deletion_full} reports Region Deletion AUC for all six datasets and four backbones (24 configurations, 200 synchronized evaluation samples each). The pattern matches the main text subset: GradCAM is strongest overall (mean 0.606 over 4{,}784 image and model evaluations, against 0.597 for GradCAM++, 0.594 for \nobsp{}-CAM, and 0.572 for Eigen-CAM), with \nobsp{}-CAM close throughout and attaining the row maximum on MobileNetV2 with CIFAR-10 and EfficientNet-B0 with CIFAR-100.

\begin{table}[t]
\centering
\caption{Region Deletion AUC on the full vision profile (higher is better). Values are means over 200 synchronized image samples per row. Bold marks the numerical maximum per row, not statistical significance.}
\label{tab:vision_region_deletion_full}
\scriptsize
\setlength{\tabcolsep}{2pt}
\begin{tabular}{@{}llcccc@{}}
\toprule
Model & Dataset & GradCAM & GradCAM++ & Eigen-CAM & \nobsp{} \\
\midrule
ResNet-18 & Imagenette & \textbf{0.538} & 0.534 & 0.474 & 0.533 \\
ResNet-18 & CIFAR-10 & 0.629 & 0.610 & \textbf{0.647} & 0.611 \\
ResNet-18 & CIFAR-100 & 0.738 & 0.742 & \textbf{0.768} & 0.746 \\
ResNet-18 & TinyImageNet & \textbf{0.718} & 0.699 & 0.632 & 0.697 \\
ResNet-18 & STL10 & \textbf{0.611} & 0.604 & 0.600 & 0.601 \\
ResNet-18 & DermaMNIST & \textbf{0.582} & 0.552 & 0.534 & 0.549 \\
MobileNetV2 & Imagenette & \textbf{0.453} & 0.439 & 0.403 & 0.410 \\
MobileNetV2 & CIFAR-10 & 0.601 & 0.604 & 0.580 & \textbf{0.614} \\
MobileNetV2 & CIFAR-100 & \textbf{0.765} & 0.754 & 0.740 & 0.758 \\
MobileNetV2 & TinyImageNet & \textbf{0.659} & \textbf{0.659} & 0.635 & 0.656 \\
MobileNetV2 & STL10 & \textbf{0.527} & 0.515 & 0.492 & 0.513 \\
MobileNetV2 & DermaMNIST & \textbf{0.553} & 0.550 & 0.533 & 0.535 \\
EfficientNet-B0 & Imagenette & \textbf{0.444} & 0.432 & 0.372 & 0.424 \\
EfficientNet-B0 & CIFAR-10 & \textbf{0.551} & 0.549 & 0.521 & 0.540 \\
EfficientNet-B0 & CIFAR-100 & 0.721 & 0.719 & 0.693 & \textbf{0.723} \\
EfficientNet-B0 & TinyImageNet & \textbf{0.720} & 0.709 & 0.659 & 0.703 \\
EfficientNet-B0 & STL10 & \textbf{0.567} & 0.556 & 0.499 & 0.556 \\
EfficientNet-B0 & DermaMNIST & 0.561 & \textbf{0.563} & 0.559 & 0.552 \\
RegNetY-400MF & Imagenette & \textbf{0.454} & 0.438 & 0.400 & 0.447 \\
RegNetY-400MF & CIFAR-10 & 0.593 & \textbf{0.597} & 0.587 & 0.595 \\
RegNetY-400MF & CIFAR-100 & \textbf{0.718} & 0.703 & 0.702 & 0.702 \\
RegNetY-400MF & TinyImageNet & \textbf{0.685} & 0.674 & 0.640 & 0.673 \\
RegNetY-400MF & STL10 & \textbf{0.653} & 0.645 & 0.608 & 0.642 \\
RegNetY-400MF & DermaMNIST & \textbf{0.516} & 0.478 & 0.460 & 0.478 \\
\bottomrule
\end{tabular}
\end{table}

\subsection*{C.2 Legacy Pixel Level Faithfulness Correlation}

Earlier versions of this evaluation used the Quantus Faithfulness Correlation metric, which perturbs individual pixels and correlates attribution with the resulting prediction change. Table~\ref{tab:vision_faithfulness} reproduces those results for transparency. All methods score near zero on this metric because pixel level perturbation is mismatched to the spatial granularity of class activation maps, which are produced at the resolution of the final convolutional feature map and upsampled. This mismatch motivated the region level protocol adopted in the main text: Region Deletion AUC perturbs contiguous patches at a scale commensurate with CAM resolution.

\begin{table}[t]
\centering
\caption{Legacy Faithfulness Correlation on vision benchmarks (synchronized subsets of 20--50 images per row). All methods score near zero; see the text of this section. Model 1 is ResNet-18, model 2 is EfficientNet-B0, and model 3 is RegNetY-400MF.}
\label{tab:vision_faithfulness}
\footnotesize
\setlength{\tabcolsep}{2pt}
\begin{tabular}{@{}llcccc@{}}
\toprule
Model & Dataset & GradCAM & GradCAM++ & Eigen-CAM & \nobsp{} \\
\midrule
\multirow{3}{*}{1}
  & CIFAR-10 & 0.041 & 0.048 & 0.021 & 0.011 \\
  & TinyImageNet & -0.004 & -0.002 & -0.035 & 0.023 \\
  & DermaMNIST & 0.040 & 0.032 & 0.010 & 0.023 \\
\midrule
\multirow{3}{*}{2}
  & CIFAR-10 & 0.037 & 0.039 & 0.030 & 0.010 \\
  & TinyImageNet & 0.058 & 0.069 & 0.025 & 0.060 \\
  & DermaMNIST & -0.016 & -0.035 & -0.020 & -0.036 \\
\midrule
\multirow{3}{*}{3}
  & CIFAR-10 & 0.007 & 0.012 & 0.026 & 0.050 \\
  & TinyImageNet & -0.018 & 0.014 & 0.022 & -0.028 \\
  & DermaMNIST & 0.057 & 0.025 & 0.019 & 0.058 \\
\bottomrule
\end{tabular}
\end{table}


\section*{Supplementary Material D: Extended Tabular Results}

\subsection*{D.1 Per Feature Function Reproduction Scores}

Table~\ref{tab:frs_per_feature} reports the direct FRS R$^2$ of each method on each nonzero component of the synthetic additive benchmark, together with the mean absolute attribution on the null feature. \nobsp{} attains the highest recovery on every component, with the largest margins on the cubic and sinusoidal components; Integrated Gradients degrades most on the cubic component.

\begin{table}[t]
\centering
\caption{Per feature direct FRS R$^2$ (higher is better) and null feature mean absolute attribution (lower is better) on the synthetic additive benchmark, mean $\pm$ SD over seeds 0--2. Bold marks the numerical best per row.}
\label{tab:frs_per_feature}
\footnotesize
\setlength{\tabcolsep}{3pt}
\begin{tabular}{@{}lccc@{}}
\toprule
Component & \nobsp{} & KernelSHAP & IG \\
\midrule
$|x_1|$ & \textbf{0.985 $\pm$ 0.003} & 0.974 $\pm$ 0.004 & 0.933 $\pm$ 0.006 \\
$x_2^3$ & \textbf{0.986 $\pm$ 0.002} & 0.952 $\pm$ 0.004 & 0.839 $\pm$ 0.036 \\
$e^{x_3}$ & \textbf{0.999 $\pm$ 0.001} & 0.995 $\pm$ 0.001 & 0.984 $\pm$ 0.003 \\
$\sin(2 x_4)$ & \textbf{0.987 $\pm$ 0.005} & 0.941 $\pm$ 0.007 & 0.932 $\pm$ 0.005 \\
Null $x_5$ (abs.) & \textbf{0.009 $\pm$ 0.003} & 0.014 $\pm$ 0.001 & 0.022 $\pm$ 0.003 \\
\bottomrule
\end{tabular}
\end{table}

\subsection*{D.2 Local Conditional Pairwise Delta Agreement}

For real datasets, ground truth component functions are unavailable. As a diagnostic companion, we introduce a local conditional pairwise delta agreement (LCPDA) analysis, inspired by local fidelity and infidelity style evaluation. For pairs of test samples that share the predicted class (classification only), differ in one feature by at least the median nonzero delta, and are close in all other features (15th percentile distance threshold, at most 800 pairs per feature), we compare the model output delta with the explanation delta assigned to that feature, and report the Pearson correlation and normalized RMSE between the two across pairs. \nobsp{} is calibrated on the training split (up to 1{,}000 samples); Integrated Gradients uses a mean baseline with 32 steps; KernelSHAP uses 40 background samples and 100 coalition samples. Values are mean $\pm$ SD over model seeds 0--2.

LCPDA is not a standard benchmark metric, and we do not use it as evidence of superiority. Results are mixed (Table~\ref{tab:lcpda}): by Pearson agreement, \nobsp{} is strongest on Iris and Breast Cancer, Integrated Gradients ties KernelSHAP on Adult, and KernelSHAP is strongest on Wine, German Credit, Bank Marketing, Covertype, Diabetes, and California Housing. We report it alongside the standard metrics and the synthetic ground truth FRS for transparency.

\begin{table*}[t]
\centering
\caption{LCPDA on real tabular datasets: Pearson correlation between model output deltas and explanation deltas (higher is better) and normalized RMSE (lower is better), mean $\pm$ SD over seeds 0--2. Bold marks the numerical best per row and metric.}
\label{tab:lcpda}
\scriptsize
\setlength{\tabcolsep}{2pt}
\begin{tabular}{@{}lcccccc@{}}
\toprule
& \multicolumn{3}{c}{Pearson $\uparrow$} & \multicolumn{3}{c}{NRMSE $\downarrow$} \\
\cmidrule(lr){2-4} \cmidrule(lr){5-7}
Dataset & \nobsp{} & IG & KSHAP & \nobsp{} & IG & KSHAP \\
\midrule
Iris & \textbf{0.503 $\pm$ 0.020} & 0.429 $\pm$ 0.025 & 0.484 $\pm$ 0.024 & \textbf{0.867 $\pm$ 0.013} & 0.935 $\pm$ 0.033 & 0.881 $\pm$ 0.021 \\
Wine & 0.254 $\pm$ 0.016 & 0.228 $\pm$ 0.012 & \textbf{0.275 $\pm$ 0.009} & 0.976 $\pm$ 0.002 & 0.992 $\pm$ 0.009 & \textbf{0.970 $\pm$ 0.004} \\
Breast Cancer & \textbf{0.434 $\pm$ 0.017} & 0.401 $\pm$ 0.014 & 0.396 $\pm$ 0.008 & 0.938 $\pm$ 0.004 & \textbf{0.933 $\pm$ 0.008} & 0.947 $\pm$ 0.007 \\
Adult & 0.629 $\pm$ 0.035 & \textbf{0.862 $\pm$ 0.027} & \textbf{0.862 $\pm$ 0.026} & 0.787 $\pm$ 0.027 & \textbf{0.509 $\pm$ 0.041} & 0.522 $\pm$ 0.037 \\
German Credit & 0.204 $\pm$ 0.012 & 0.216 $\pm$ 0.013 & \textbf{0.259 $\pm$ 0.006} & 0.983 $\pm$ 0.003 & 0.999 $\pm$ 0.004 & \textbf{0.968 $\pm$ 0.001} \\
Bank Marketing & 0.417 $\pm$ 0.020 & 0.434 $\pm$ 0.031 & \textbf{0.522 $\pm$ 0.023} & 0.938 $\pm$ 0.027 & 0.945 $\pm$ 0.029 & \textbf{0.869 $\pm$ 0.013} \\
Covertype & 0.297 $\pm$ 0.015 & 0.313 $\pm$ 0.029 & \textbf{0.384 $\pm$ 0.029} & 1.032 $\pm$ 0.032 & 1.088 $\pm$ 0.031 & \textbf{0.947 $\pm$ 0.018} \\
Diabetes & 0.519 $\pm$ 0.165 & 0.630 $\pm$ 0.002 & \textbf{0.658 $\pm$ 0.004} & 0.951 $\pm$ 0.266 & 0.777 $\pm$ 0.002 & \textbf{0.765 $\pm$ 0.002} \\
California Housing & 0.668 $\pm$ 0.081 & 0.600 $\pm$ 0.010 & \textbf{0.745 $\pm$ 0.038} & 0.772 $\pm$ 0.082 & 0.948 $\pm$ 0.006 & \textbf{0.700 $\pm$ 0.038} \\
\bottomrule
\end{tabular}
\end{table*}